%% file: MAIN.tex
\documentclass{article}

\usepackage{arxiv}

\usepackage[utf8]{inputenc} % allow utf-8 input
\usepackage[T1]{fontenc}    % use 8-bit T1 fonts
\usepackage{hyperref}       % hyperlinks
\usepackage{url}            % simple URL typesetting
\usepackage{booktabs}       % professional-quality tables
\usepackage{nicefrac}       % compact symbols for 1/2, etc.
\usepackage{microtype}      % microtypography
\usepackage{lipsum}
\usepackage{graphicx}

\usepackage{microtype}
\usepackage{graphicx}
\usepackage{subfigure}
\usepackage{booktabs} % for professional tables
\usepackage{placeins}
\usepackage[utf8]{inputenc} % allow utf-8 input
\usepackage[T1]{fontenc}    % use 8-bit T1 fonts
\usepackage{hyperref}       % hyperlinks
\usepackage{url}            % simple URL typesetting
\usepackage{booktabs}       % professional-quality tables
\usepackage{nicefrac}       % compact symbols for 1/2, etc.
\usepackage{microtype}      % microtypography
\usepackage{lipsum}
\usepackage{fancyhdr}       % header
\usepackage{graphicx}       % graphics
\graphicspath{{media/}}     % organize your images and other figures under media/ folder

\usepackage{makecell}

\usepackage{xcolor,colortbl}
\definecolor{LightCyan}{rgb}{0.88,1,1}
\definecolor{LightPurple}{rgb}{0.88,0.88,1}
\definecolor{LightGray}{rgb}{0.93, 0.93, 0.94}

\newcolumntype{x}{>{\columncolor{LightCyan}} c}
\newcolumntype{y}{>{\columncolor{LightPurple}} c}
\newcolumntype{z}{>{\columncolor{LightGray}} c}
\usepackage{nicematrix}
\usepackage{listings}
\usepackage{parcolumns}
\usepackage{pifont}% http://ctan.org/pkg/pifont
\usepackage{newtxtext,newtxmath}
\usepackage{xcolor}

\usepackage{mathtools}

\usepackage{mathtools} 
\usepackage{empheq}

\newcommand{\norm}[1]{\left\lVert#1\right\rVert}
\newcommand{\AMatrix}{\mathbf{A}}  
 
\newcommand{\IMatrix}{\mathbf{I}}
\newcommand{\DMatrix}{\mathbf{D}}
\newcommand{\LMatrix}{\mathbf{L}}

\newcommand{\UMatrix}{\mathbf{U}}
\newcommand{\VMatrix}{\mathbf{V}}

\newcommand{\PMatrix}{\mathbf{P}}
\newcommand{\HMatrix}{\mathbf{H}}
\newcommand{\XMatrix}{\mathbf{X}}
\newcommand{\OmegaMatrix}{\mathbf{\Omega}}

\newcommand{\YMatrix}{\mathbf{Y}}
\newcommand{\ZMatrix}{\mathbf{Z}}

\newcommand{\QMatrix}{\mathbf{Q}}
\newcommand{\RMatrix}{\mathbf{R}}
\newcommand{\TMatrix}{\mathbf{T}}
\newcommand{\EMatrix}{\mathbf{E}}
\newcommand{\AlmostZero}{\mathbf{\widetilde{O}}}

\newcommand{\identityvec}{\mathbf{e}}
\newcommand{\qvec}{\mathbf{q}}
\newcommand{\ortvec}{\mathbf{v}}

\newcommand{\algoArnoldi}{{\sc Arnoldi-GCN}}
\newcommand{\algoVandermonde}{{\sc Vandermonde-GCN}}
\newcommand{\algoGArnoldi}{{\sc G-Arnoldi-GCN}}
\newcommand{\nodeset}{\mathcal{V}}
\newcommand{\edgeset}{\mathcal{E}}
\newcommand{\network}{\mathcal{G}}

\newcommand{\krylov}{\mathcal{K}}

\usepackage{subfigure}
\usepackage{algorithm}% http://ctan.org/pkg/algorithms
\usepackage{caption}
\DeclareCaptionType{equ}[][]

\newtheorem{proof}{Proof}%
\newtheorem{theorem}{Theorem}

\ExplSyntaxOn

\cs_new_protected:Nn \ben_matrix_preamble:nn
 {
  \seq_set_split:Nnn \l_tmpa_seq { | } { #2 }
  \seq_set_map:NNn \l_tmpb_seq \l_tmpa_seq { *{##1}{#1} }
  \exp_args:Ne \array{ @{} \seq_use:Nn \l_tmpb_seq { | } @{} }
 }

\clist_map_inline:nn { {},b,B,p,v,V }
 {
  \use:e
   {
    \NewDocumentEnvironment{#1matrix+}{O{c}m}
     {
      \left\str_case:nnF { #1 } {{b}{[}{B}{\{}{p}{(}{v}{|}{V}{\|}}{.}
      \ben_matrix_preamble:nn { ##1 } { ##2 }
     }
     {
      \exp_not:N \endarray
      \right\str_case:nnF { #1 } {{b}{]}{B}{\}}{p}{)}{v}{|}{V}{\|}}{.}
     }
   }
 }

\ExplSyntaxOff

\newtheorem{definition}{Definition}

\usepackage{hyperref}

\usepackage{cite}
\usepackage{amsmath}

\usepackage{algorithmic}
\usepackage{graphicx}
\usepackage{textcomp}
\usepackage{xcolor}
\def\BibTeX{{\rm B\kern-.05em{\sc i\kern-.025em b}\kern-.08em
    T\kern-.1667em\lower.7ex\hbox{E}\kern-.125emX}}

\title{Generalized Learning of Coefficients in Spectral Graph Convolutional Networks}

\author{
 Mustafa Coskun \\
  Artificial Intelligence Department\\
  Ankara University\\
  Ankara, 06033, TURKEY \\
  \texttt{coskunmustafa@ankara.edu.tr} \\
   \And
 Ananth Grama \\
  Department of Computer Science\\
  Purdue University\\
  West Lafayette, IN 47906, USA \\
  \texttt{ayg@cs.purdue.edu} \\
  \And
 Mehmet Koyut\"urk \\
  Department ofComputer and Data Sciences\\
  Case Western Reserve University\\
  Cleveland, OH 44106, USA \\
  \texttt{mehmet.koyuturk@case.edu} \\
}

\begin{document}
\maketitle
\begin{abstract}
 Spectral Graph Convolutional Networks (GCNs) have gained popularity in graph machine learning applications due, in part, to their flexibility in specification of network propagation rules.
%, including node classification. 
These propagation rules are often constructed as polynomial filters
whose coefficients are learned using label information during training.
%to define network propagation rules and learn the coefficients of the polynomials using label information during training. 
In contrast to learned polynomial filters, explicit filter functions are useful in capturing relationships between network topology and distribution of labels across the network. %(e.g., personalized page rank for homophily vs. low-pass filter for heterophily). 
A number of algorithms incorporating either approach have been proposed; however the relationship between filter functions and polynomial approximations is not fully resolved.
This is largely due to the ill-conditioned nature of the linear systems that must be solved to derive polynomial approximations of filter functions. To address this challenge, we propose a novel Arnoldi orthonormalization-based algorithm, along with a unifying approach, called \algoGArnoldi{} that can efficiently and effectively approximate a given filter function with a polynomial.  
We evaluate \algoGArnoldi{} in the context of multi-class node classification across ten datasets with diverse topological characteristics. 
%To assess its customizability, we explore various combinations of sampling methods and filter functions. 
Our experiments show that \algoGArnoldi{} consistently outperforms state-of-the-art methods when suitable filter functions are employed. Overall, \algoGArnoldi{} opens important new directions in graph machine learning by enabling the explicit design and application of diverse filter functions. Code link: https://github.com/mustafaCoskunAgu/GArnoldi-GCN
\end{abstract}

% keywords can be removed

\input{Introduction}
\input{Method}

\input{Results}

\input{Conclusion}
\input{Appendix}
\bibliographystyle{unsrt}  
\bibliography{references}  %%% Remove comment to use the external .bib file (using bibtex).

\end{document}

%% file: Introduction.tex
\section{Introduction}
\label{sec:intro}

Spectral Graph Convolutional Networks (Spectral-GCNs) have attracted significant attention in various graph representation learning tasks, including node classification~\cite{he2021bernnet}, link prediction~\cite{schlichtkrull2018modeling,cocskun2021node,cocskun2021fast}, and applications like drug discovery~\cite{rathi2019practical, jiang2021could}. Spectral-GCNs utilize spectral convolution or spectral graph filters, operating in the spectral domain of the graph Laplacian matrix or normalized adjacency matrix~\cite{chien2020adaptive, he2021bernnet} to address the "feature-over-smoothing" problem by isolating the propagation scheme from neural networks. 
Spectral GCNs enable distribution of label/ feature information beyond
neighboring nodes, as in Spatial GCNs~\cite{kipf2016semi}, to distant nodes~\cite{gasteiger2018predict}. A key element of 
Spectral GCNs is spectral graph convolutions or filters, which control the distribution of processed label/feature information. These convolutions assign weights to contributions from each hop in the graph, and the challenge in Spectral GCNs revolves around finding optimal weights.

 Spectral GCNs are grouped into two major categories: predetermined spectral graph convolution/filters~\cite{gasteiger2018predict,zhu2020simple, zhu2021interpreting} and learnable spectral graph convolution/filters, such as ChebNet~\cite{defferrard2016convolutional}, CayleyNet~\cite{levie2018cayleynets}, GPR-GNN~\cite{chien2020adaptive}, BernNet~\cite{he2021bernnet}, JacobiCon~\cite{wang2022powerful}. Applications in the former class rely on domain/ data knowledge to assign polynomial coefficients to design custom filters, whereas those in the latter class learn coefficients using training labels.
%
% \textbf{Predetermined spectral graph convolution/filters:} These Spectral GCNs employ an implicit filter with fixed weights for propagation using a Monomial polynomial. For instance, APPNP~\cite{gasteiger2018predict} utilizes Personalized PageRank (PPR) with an implicit filter of $\frac{1}{1-\alpha}$, forming a Monomial polynomial sequence ${1,\alpha, \alpha^2, \cdot, \cdot, \cdot}$, where $\alpha \in (0,1)$ represents the damping factor in PPR. Subsequently, similar fixed implicit filter-based algorithms like Markov Diffusion Kernel~\cite{zhu2020simple}, parameterized random walk~\cite{zhu2021interpreting}, etc., have been proposed, all employing fixed coefficients for propagation via implicit filters and Monomial polynomials.
%
%  \textbf{Learnable spectral graph convolution/filters}: These Spectral GCNs learn the weights to assign to different hops using neural networks. They initially choose a specific polynomial (e.g., Chebyshev) with fixed coefficients and then adapt these coefficients based on label and feature information within the neural network, leveraging the chosen polynomial's update rule, such as the three-term recurrence for Chebyshev polynomials. That is, they substitute filter functions with polynomials. Notable examples of such Spectral GCNs include ChebNet~\cite{defferrard2016convolutional}, CayleyNet~\cite{levie2018cayleynets}, GPR-GNN~\cite{chien2020adaptive}, BernNet~\cite{he2021bernnet}, JacobiCon~\cite{wang2022powerful}, employing Chebyshev, Cayley, Monomial, Bernstein, and Jacobi Polynomials, respectively, for their Spectral GCN propagation.
%
Despite the effectiveness of current Spectral GCNs, their methodologies often come with inherent challenges. Specifically, these models either depend on predetermined coefficients, embodying an implicit filter  (e.g., APPNP~\cite{gasteiger2018predict}, GNN-LP/HP~\cite{zhu2021interpreting}, etc.), or utilize a specific polynomial and its update rules without establishing a connection with any filter function~\cite{defferrard2016convolutional, levie2018cayleynets,he2021bernnet,wang2022powerful}.
 In such cases, models rely on neural networks to learn an arbitrary filter based on label/feature information and polynomials' update rules. However, the lack of a principled link between a filter and the selected polynomial introduces uncertainties regarding the assertion of effectively learning an arbitrary filter.

Our focus in this paper is on the design of polynomials as proxies for implementing explicitly specified filter functions. These polynomials specify the propagation rules, enabling interpretation and effective design of filter functions. Our goal is to establish the mathematical foundations to enable their joint exploration for design of effective and computationally efficient filters. We achieve this goal through the following major contributions:

\begin{enumerate}
    \item We approximate a given filter to a polynomial by formulating and solving the corresponding Vandermonde linear system. We show that the exponential condition number of the Vandermonde matrix leads to the computation of invalid polynomial coefficients, resulting in poor polynomial approximations.
    \item We demonstrate that an alternate QR decomposition for the  Vandermonde matrix can be computed through an Arnoldi-process. This approach produces precise polynomial approximations for any filter.
    \item Using polynomial approximations obtained via Arnoldi, we introduce the \algoArnoldi{} algorithm, eliminating the need for polynomial update rules in current methods. This offers a powerful new propagation scheme independent of the polynomial.
    \item Finally, we extend the \algoArnoldi\ approach to \algoGArnoldi\ by refining polynomial coefficients computed by \algoArnoldi\ within neural networks using label information, thereby broadening its applicability within the context of neural networks.
\end{enumerate}

We provide comprehensive theoretical underpinnings for our findings and empirically demonstrate the efficiency of our approaches across diverse real-world benchmark datasets and problems. First, we approximate eight representative filter functions, including simple random walks and intricate band-pass rejection filters, using Equispaced, Chebyshev, Legendre, and Jacobi polynomial samples. 
We then assess node classification performance of Spectral GCNs that are built using these approximations.
Our results show that our Arnoldi-based approximation consistently outperforms the direct solution of the Vandermonde system in generating precise polynomial approximations across all filter functions. 
%Furthermore, when we integrate these approximations into Spectral GCN contexts,  we obtain superior approximations significantly enhance node classification performance, specifically with \algoArnoldi{}. Extending the effectiveness validation, we generalize \algoArnoldi{} to \algoGArnoldi{}, refining polynomial approximations within neural networks using label information. 
Rigorous experiments on 15 benchmark datasets, covering semi-supervised and fully supervised learning tasks, reveal that our algorithms significantly advance state-of-the-art methods in node classification performance.

% Key observations from the experimental results include:
% \begin{enumerate}
%     \item Simple filters paired with any polynomial sampling yield optimal outcomes for homophilic datasets, while complex filters combined with Jacobi and Legendre polynomial samples result in superior results for heterophilic datasets in semi-supervised learning tasks. Simultaneously, simple filters generally demonstrate higher accuracy in full supervised learning scenarios.
%     \item The choice of the best polynomial sampling varies depending on the dataset and filter used, which have been significant challenge within the Spectral GCN community so far.
%     \item \algoGArnoldi{} consistently outperforms its plain version, \algoArnoldi{}, highlighting the beneficial utilization of label information in our approach.
% \end{enumerate}

% These observations highlight the versatility and effectiveness of our methodology across diverse datasets and learning scenarios, emphasizing its potential to significantly improve node classification performance within the Spectral GCN domain.

\begin{figure*}[ht]
	\centering
	\scalebox{0.45}{\input{ArnoldiGCN.pspdftex}}
	\caption{Workflow of proposed Arnoldi-Based Generalized Spectral GCN. }  
\label{fig:MainFig}
\end{figure*}
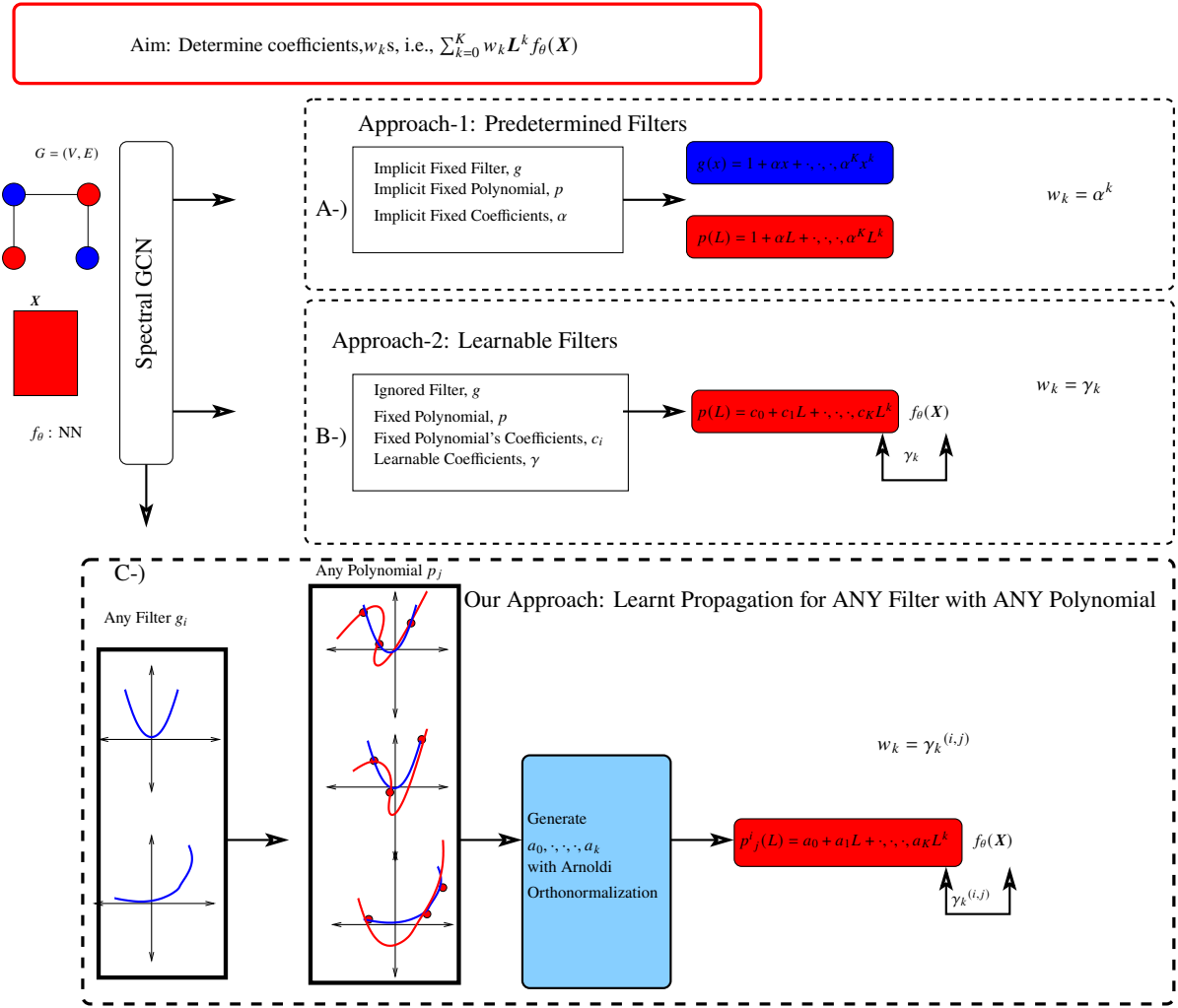

%% file: ArnoldiGCN.pspdftex
\begin{picture}(0,0)%
\includegraphics{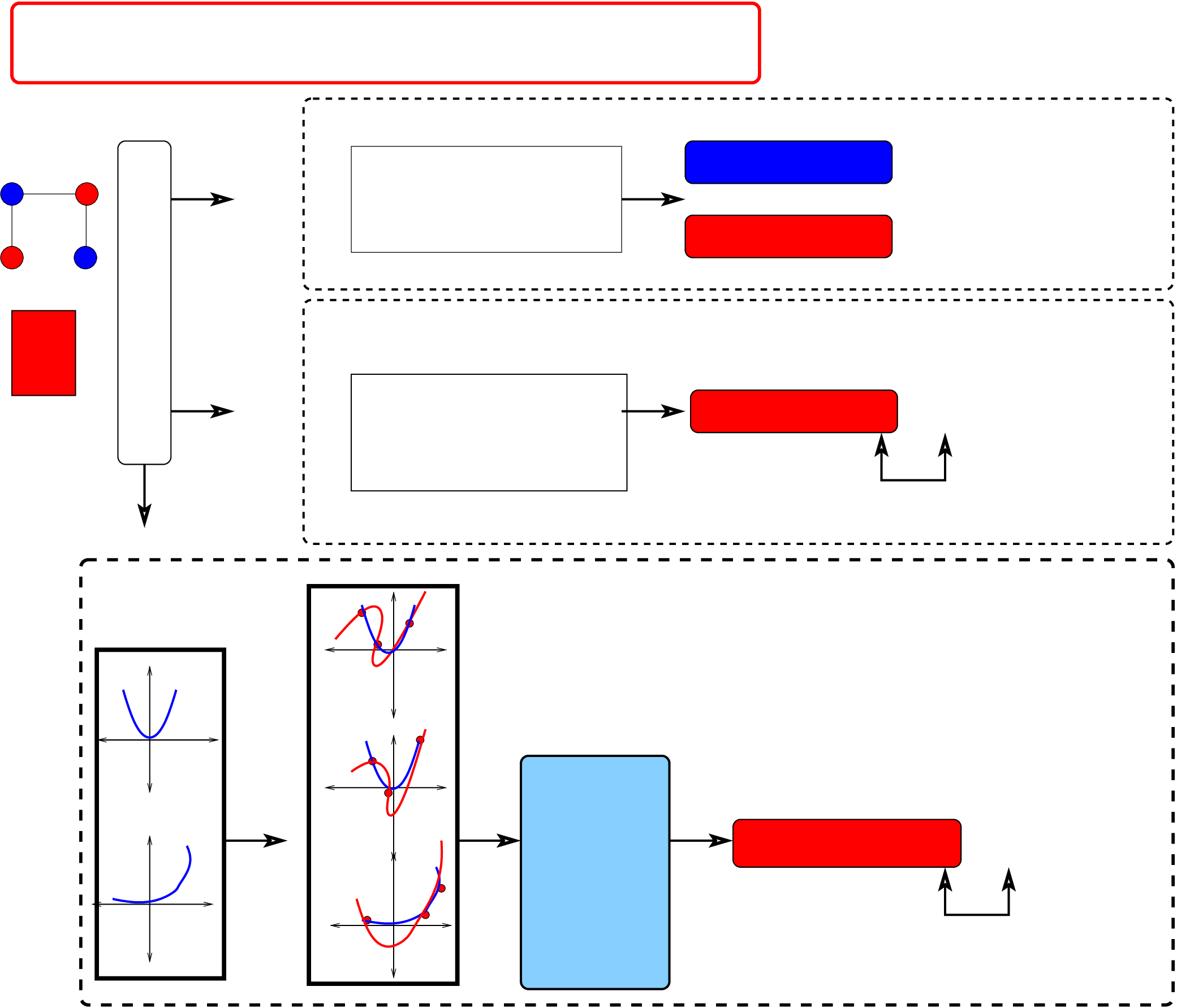}%
\end{picture}%
\setlength{\unitlength}{3947sp}%
\begingroup\makeatletter\ifx\SetFigFont\undefined%
\gdef\SetFigFont#1#2#3#4#5{%
  \reset@font\fontsize{#1}{#2pt}%
  \fontfamily{#3}\fontseries{#4}\fontshape{#5}%
  \selectfont}%
\fi\endgroup%
\begin{picture}(16637,14263)(133,-13530)
\put(6676,-7861){\makebox(0,0)[lb]{\smash{{\SetFigFont{20}{24.0}{\rmdefault}{\mddefault}{\updefault}{\color[rgb]{0,0,0}Our Approach: Learnt Propagation for ANY Filter with ANY Polynomial}%
}}}}
\put(601,-1486){\makebox(0,0)[lb]{\smash{{\SetFigFont{12}{14.4}{\rmdefault}{\mddefault}{\updefault}{\color[rgb]{0,0,0}$G = (V, E)$}%
}}}}
\put(526,-3586){\makebox(0,0)[lb]{\smash{{\SetFigFont{12}{14.4}{\rmdefault}{\mddefault}{\updefault}{\color[rgb]{0,0,0}$\boldsymbol{X}$}%
}}}}
\put(526,-5461){\makebox(0,0)[lb]{\smash{{\SetFigFont{14}{16.8}{\rmdefault}{\mddefault}{\updefault}{\color[rgb]{0,0,0}$f_{\theta}$ : NN}%
}}}}
\put(2251,-4486){\rotatebox{90.0}{\makebox(0,0)[lb]{\smash{{\SetFigFont{20}{24.0}{\rmdefault}{\mddefault}{\updefault}{\color[rgb]{0,0,0}Spectral GCN}%
}}}}}
\put(14926,-2086){\makebox(0,0)[lb]{\smash{{\SetFigFont{17}{20.4}{\rmdefault}{\mddefault}{\updefault}{\color[rgb]{0,0,0}$w_k = \alpha^k$}%
}}}}
\put(9976,-1636){\makebox(0,0)[lb]{\smash{{\SetFigFont{14}{16.8}{\rmdefault}{\mddefault}{\updefault}{\color[rgb]{0,0,0}$g(x) = 1 + \alpha x + \cdot,\cdot, \cdot, \alpha^K x^k$}%
}}}}
\put(9976,-2686){\makebox(0,0)[lb]{\smash{{\SetFigFont{14}{16.8}{\rmdefault}{\mddefault}{\updefault}{\color[rgb]{0,0,0}$p(L) = 1 + \alpha L + \cdot,\cdot, \cdot, \alpha^K L^k$}%
}}}}
\put(5401,-2386){\makebox(0,0)[lb]{\smash{{\SetFigFont{14}{16.8}{\rmdefault}{\mddefault}{\updefault}{\color[rgb]{0,0,0}Implicit Fixed Coefficients, $\alpha$}%
}}}}
\put(5401,-2011){\makebox(0,0)[lb]{\smash{{\SetFigFont{14}{16.8}{\rmdefault}{\mddefault}{\updefault}{\color[rgb]{0,0,0}Implicit Fixed Polynomial, $p$}%
}}}}
\put(5401,-1711){\makebox(0,0)[lb]{\smash{{\SetFigFont{14}{16.8}{\rmdefault}{\mddefault}{\updefault}{\color[rgb]{0,0,0}Implicit Fixed Filter, $g$}%
}}}}
\put(4576,-2311){\makebox(0,0)[lb]{\smash{{\SetFigFont{20}{24.0}{\rmdefault}{\mddefault}{\updefault}{\color[rgb]{0,0,0}A-)}%
}}}}
\put(5401,-5236){\makebox(0,0)[lb]{\smash{{\SetFigFont{14}{16.8}{\rmdefault}{\mddefault}{\updefault}{\color[rgb]{0,0,0}Fixed Polynomial, $p$}%
}}}}
\put(5401,-5836){\makebox(0,0)[lb]{\smash{{\SetFigFont{14}{16.8}{\rmdefault}{\mddefault}{\updefault}{\color[rgb]{0,0,0}Learnable Coefficients, $\gamma$}%
}}}}
\put(5401,-5536){\makebox(0,0)[lb]{\smash{{\SetFigFont{14}{16.8}{\rmdefault}{\mddefault}{\updefault}{\color[rgb]{0,0,0}Fixed Polynomial's Coefficients, $c_{i}$}%
}}}}
\put(4576,-5536){\makebox(0,0)[lb]{\smash{{\SetFigFont{20}{24.0}{\rmdefault}{\mddefault}{\updefault}{\color[rgb]{0,0,0}B-)}%
}}}}
\put(9976,-5161){\makebox(0,0)[lb]{\smash{{\SetFigFont{14}{16.8}{\rmdefault}{\mddefault}{\updefault}{\color[rgb]{0,0,0}$p(L) = c_0 + c_1 L + \cdot,\cdot, \cdot, c_K L^k$}%
}}}}
\put(12976,-5161){\makebox(0,0)[lb]{\smash{{\SetFigFont{14}{16.8}{\rmdefault}{\mddefault}{\updefault}{\color[rgb]{0,0,0}$f_\theta (\boldsymbol{X})$}%
}}}}
\put(14776,-4786){\makebox(0,0)[lb]{\smash{{\SetFigFont{17}{20.4}{\rmdefault}{\mddefault}{\updefault}{\color[rgb]{0,0,0}$w_k = \gamma_{k}$}%
}}}}
\put(12901,-5761){\makebox(0,0)[lb]{\smash{{\SetFigFont{14}{16.8}{\rmdefault}{\mddefault}{\updefault}{\color[rgb]{0,0,0}$\gamma_{k}$}%
}}}}
\put(1576,-8086){\makebox(0,0)[lb]{\smash{{\SetFigFont{14}{16.8}{\rmdefault}{\mddefault}{\updefault}{\color[rgb]{0,0,0}Any Filter $g_i$}%
}}}}
\put(7576,-10936){\makebox(0,0)[lb]{\smash{{\SetFigFont{14}{16.8}{\rmdefault}{\mddefault}{\updefault}{\color[rgb]{0,0,0}Generate }%
}}}}
\put(7576,-11611){\makebox(0,0)[lb]{\smash{{\SetFigFont{14}{16.8}{\rmdefault}{\mddefault}{\updefault}{\color[rgb]{0,0,0}with Arnoldi}%
}}}}
\put(7576,-11986){\makebox(0,0)[lb]{\smash{{\SetFigFont{14}{16.8}{\rmdefault}{\mddefault}{\updefault}{\color[rgb]{0,0,0}Orthonormalization}%
}}}}
\put(10576,-11236){\makebox(0,0)[lb]{\smash{{\SetFigFont{14}{16.8}{\rmdefault}{\mddefault}{\updefault}{\color[rgb]{0,0,0}${p^i}_j(L) = a_0 + a_1 L + \cdot,\cdot, \cdot, a_K L^k$}%
}}}}
\put(13576,-12061){\makebox(0,0)[lb]{\smash{{\SetFigFont{14}{16.8}{\rmdefault}{\mddefault}{\updefault}{\color[rgb]{0,0,0}${\gamma_{k}}^{(i,j)}$}%
}}}}
\put(13876,-11236){\makebox(0,0)[lb]{\smash{{\SetFigFont{14}{16.8}{\rmdefault}{\mddefault}{\updefault}{\color[rgb]{0,0,0}$f_\theta (\boldsymbol{X})$}%
}}}}
\put(12526,-9886){\makebox(0,0)[lb]{\smash{{\SetFigFont{17}{20.4}{\rmdefault}{\mddefault}{\updefault}{\color[rgb]{0,0,0}$w_k = {\gamma_{k}}^{(i,j)}$}%
}}}}
\put(7576,-11311){\makebox(0,0)[lb]{\smash{{\SetFigFont{14}{16.8}{\rmdefault}{\mddefault}{\updefault}{\color[rgb]{0,0,0}$a_0, \cdot, \cdot,\cdot, a_k$}%
}}}}
\put(4576,-7411){\makebox(0,0)[lb]{\smash{{\SetFigFont{14}{16.8}{\rmdefault}{\mddefault}{\updefault}{\color[rgb]{0,0,0}Any Polynomial $p_j$}%
}}}}
\put(1726,-7486){\makebox(0,0)[lb]{\smash{{\SetFigFont{20}{24.0}{\rmdefault}{\mddefault}{\updefault}{\color[rgb]{0,0,0}C-)}%
}}}}
\put(1951, 14){\makebox(0,0)[lb]{\smash{{\SetFigFont{17}{20.4}{\rmdefault}{\mddefault}{\updefault}{\color[rgb]{0,0,0}Aim: Determine coefficients,$w_k$s, i.e., $\sum_{k=0}^K w_k \boldsymbol{L}^k f_{\theta} (\boldsymbol{X})$}%
}}}}
\put(5401,-4861){\makebox(0,0)[lb]{\smash{{\SetFigFont{14}{16.8}{\rmdefault}{\mddefault}{\updefault}{\color[rgb]{0,0,0}Ignored Filter, $g$}%
}}}}
\put(5176,-1111){\makebox(0,0)[lb]{\smash{{\SetFigFont{20}{24.0}{\rmdefault}{\mddefault}{\updefault}{\color[rgb]{0,0,0}Approach-1: Predetermined Filters}%
}}}}
\put(4801,-4186){\makebox(0,0)[lb]{\smash{{\SetFigFont{20}{24.0}{\rmdefault}{\mddefault}{\updefault}{\color[rgb]{0,0,0}Approach-2: Learnable Filters}%
}}}}
\end{picture}%

%% file: Method.tex
\section{Background and Related Work}
\label{sec:background}

\textbf{Spectral Graph Convolutional Networks.} We denote an undirected graph by $\network = (\nodeset, \edgeset)$ with node set $\nodeset$ and edge set $\edgeset$, where the number of nodes is $|\nodeset| = n$. % and the number of edge is $|\edgeset|$. 
In the context of node classification, we are given an $n\times m$ feature matrix $\XMatrix$, where $m$ denotes the number of features and $\XMatrix(i,j)$ represents the value of the $j$th feature for the $i$th node. 
We are also given an $n \times c$-dimensional label matrix $\YMatrix$, where $c$ denotes the number of classes, and $\YMatrix(i,j)$ indicates whether node $i$ belongs to class $j$.
Our goal is to learn a neural network model $\YMatrix = f(\network; \XMatrix)$, where $f$ utilizes topological relationships between nodes in $\network$ to
enhance the generalizibility of the machine learning 
model. 
%{\bf ayg: robustness of what model? What does robustness mean?}

Spectral GCNs use spectral filter functions to propagate signals across the graph. 
The signals include input features and latent features that are computed at intermediate layers of the neural network.
For the sake of generality, in the following discussion, we denote quantities propagated across the network as graph signal $\boldsymbol{x} \in \mathbb{R}^n$, where $\boldsymbol{x}(i)$ denotes the graph signal at node~$i$.

\textbf{Spectral Filter Functions.} Let $\AMatrix$ denote the adjacency matrix of  graph $\network$ and $\DMatrix$ denote the diagonal degree matrix. 
We represent the degree-normalized adjacency matrix as $\PMatrix = \DMatrix^{-1/2} \AMatrix \DMatrix^{-1/2}$ and graph Laplacian matrix as  $\LMatrix = \IMatrix - \PMatrix$, where $\IMatrix$ denotes the identity matrix. 
Let  $\LMatrix = \UMatrix \boldsymbol{\Lambda} \UMatrix^T$ be the eigendecompostion of $\LMatrix$, where $\boldsymbol{\Lambda} = \text{diag}[\lambda_1, \cdot, \cdot, \cdot, \lambda_n]$. 
%and let $\PMatrix = \VMatrix \boldsymbol{\Sigma} \VMatrix^T$ be the eigendecompostion of $\PMatrix$, where $\boldsymbol{\Sigma} = \text{diag}[\sigma_1, \cdot, \cdot, \cdot, \sigma_n]$.    

Spectral GCNs create the spectral graph convolutions, or spectral filters, in the domain of the graph Laplacian~\cite{chien2020adaptive}. Namely:
\begin{equation}
\boldsymbol{y} = 
%\UMatrix \text{diag}[g(\lambda_1), \cdot, \cdot, \cdot, g(\lambda_n)] \UMatrix^T \boldsymbol{x} = 
\UMatrix g(\boldsymbol{\Lambda}) \UMatrix^T \boldsymbol{x},
    \label{eq:filter}
\end{equation}
where $g(\boldsymbol{\Lambda}) = \text{diag}[g(\lambda_1), \cdot, \cdot, \cdot, g(\lambda_n)]$.
Here, $\boldsymbol{y}$ represents the vector obtained by filtering graph signal $\boldsymbol{x}$. The function $g(\lambda)$, called a \emph{spectral filter}, transforms the eigenvalues of the Laplacian matrix to suitably shape signal propagation over the graph. Computing the eigendecomposition in Equation~\ref{eq:filter} is computationally expensive for large matrices, motivating the use of \emph{polynomials} for approximation:
%in state-of-the-art Spectral GCN algorithms:  
%known eigenvalue bounds, (0,2] for Laplacian and [-1,1] for normalized adjacency matrix.
%Mathematically, we can rewrite Equation~\ref{eq:filter} in the polynomial sense:

\begin{equation}
\begin{aligned}
& \boldsymbol{y}  = \UMatrix g(\boldsymbol{\Lambda}) \UMatrix^T \boldsymbol{x} 
\approx \sum_{k=0}^K w_k \LMatrix^k \boldsymbol{x} 	\equiv \sum_{k=0}^K c_k \PMatrix^k \boldsymbol{x},
\end{aligned}
    \label{eq:filterPoly}
\end{equation}
% Or equivalently,
% \begin{equation}
% \boldsymbol{y} = \VMatrix \text{diag}[g(\sigma_1), \cdot, \cdot, \cdot, g(\sigma_n)] \VMatrix^T \boldsymbol{x} = \VMatrix g(\boldsymbol{\Sigma}) \VMatrix^T \boldsymbol{x} \approx \sum_{k=0}^K c_k \PMatrix^k \boldsymbol{x},
%     \label{eq:filterPolywitPMatrix}
% \end{equation}

%-------------Appendix

where $K$ is a hyperparameter that specifies the degree of the polynomial used in the approximation.
An important design criterion in Spectral GCNs is the choice of coefficients $w_k$ or $c_k$.
%, i.e., finding polynomial coefficients that represent these weights.
In the context of random walks, $K$ specifies the extent of propagation in terms of path length and coefficients $w_k$ or $c_k$ correspond to the relative weight of paths of length $k$.

\subsection{General Formulation of Spectral GCNs} 
Given a filtering function ($g(\lambda)$) or a polynomial approximation (Equation~\ref{eq:filterPoly}), the general setting for Spectral GCNs is as follows:
\begin{equation}
\begin{aligned}
\YMatrix = &~\text{softmax}\big( \ZMatrix\big), \ZMatrix = \sum_{k = 0}^K w_k \HMatrix^ {(k)},  \quad\\
\HMatrix^ {(k)} &= \LMatrix%_{\text{sym}}
\HMatrix^ {(k-1)}, \HMatrix^ {(0)} = f_{\theta} (\XMatrix)
\end{aligned}
    \label{eq:SpectralGCN}
\end{equation}
where $\XMatrix \in \mathbb{R}^{n\times m}$ denotes the feature matrix, $f_{\theta}$ denotes neural network with parameters $\theta$, 
and $\YMatrix$ denotes the label matrix.

%In the literature, there are two types of 
Current spectral GCNs
can be grouped into two categories based on how they choose the coefficients of the polynomials: Predetermined Graph Convolution and Learnable Graph Convolution. 
%{\bf ayg: can we make this terminology consistent with that in the introduction?}
%Fixed propagation-based and Learnt propagation-based Spectral GCNs.

\subsubsection{Predetermined Graph Convolution}
This class of spectral GCNs fix the weights a-priori, e.g., $c_k = \alpha^k$ where $\alpha \in (0,1)$. 
%Because of these fixed weights, these types of Spectral-GCNs are known as \textbf{Predetermined Graph Convolutions}. 
For example, \textbf{APPNP~\cite{gasteiger2018predict}} uses an implicit filter function,  $g: [-\alpha, \alpha] \rightarrow \mathbb{R}$ with $g (\omega) = \frac{1-\alpha}{(1- \omega)}$, which is approximated by the polynomial $p(\omega) = 1 + \alpha \omega + \alpha^2 \omega^2 + ,\cdot, \cdot, \cdot, $.
%by changing $\LMatrix$ to $\widetilde{\PMatrix} = \widetilde{\DMatrix}^{-1/2} \widetilde{\AMatrix} \widetilde{\DMatrix}^{-1/2}$, self-loop added normalized adjacency matrix. 
Then, starting with $\HMatrix^ {(0)} = f_{\theta} (\XMatrix)$, APPNP defines the following Spectral-GCN: 
\begin{equation}
\begin{aligned}
    &\ZMatrix_{\text{APPNP}} = \sum_{k = 0}^K \alpha^k \HMatrix^ {(k)}, \HMatrix^ {(k)} = \widetilde{\PMatrix}  \HMatrix^ {(k-1)},
\end{aligned}
    \label{eq:APPNPSpectralGCN}
\end{equation}
That is, APPNP sets $c_k=\alpha^k$.
% Specifically,
% \begin{equation}    
% \ZMatrix_{\text{APPNP}} = \HMatrix^ {(0)} + \alpha \widetilde{\PMatrix} \HMatrix^ {(1)} +  \alpha^2 {\widetilde{\PMatrix}}^2 \HMatrix^ {(2)} + ,\cdot, \cdot, \cdot
%     \label{eq:APPNPSpectralGCNWithPoly}
% \end{equation}
Similarly, \textbf{GNN-LF/HF~\cite{zhu2021interpreting}} and \textbf{SGC~\cite{wu2019simplifying}} use an implicit  filter function and approximate this function using fixed coefficients.
%$g$ and employ Monomial polynomial in their iterative process to create Spectral GCNs.  \emph{Note that Predetermined Spectral-GCNs use a single filter within the Monomial Polynomial in an implicit manner. }

 \subsubsection{Learnable Graph Convolution}
These spectral GCNs simultaneously learn the coefficients of the polynomial alongside $f_{\theta} (\XMatrix)$ by leveraging label information in training data. 
%Due to this process of learning weights (coefficients), these Spectral-GCNs are commonly referred to as \textbf{Learnable Graph Convolutions}. Below are examples illustrating these types of Spectral GCNs:
%
%
% {\color{red}\textbf{Very Important Note:} Learnable Graph Convolutions ignore/loose the connection with filters, except GPR-GNN~\cite{chien2020adaptive}, it still implicitly utilizes PPR filter, $g: [-\alpha, \alpha] \rightarrow \mathbb{R}$ with $g (x) = \frac{1}{(1- x)}$,  but it is also criticized by saying it is based on Monomial Polynomial (Interesting version may be developed by using Chopper/Chebyshev~\cite{coskun2016efficient} coefficients) .}
%
% \textbf{Examples:} Now that GPR-GNN~\cite{chien2020adaptive} is an exception, we start our discussion with GPR-GNN and then give other types \textbf{Learnable Convolutions} that ignore/by-bass filter relationships, e.g., Chebyshev, Jacobi Polynomials, etc.
%
As an example, \textbf{GPR-GNN~\cite{chien2020adaptive}} generalizes \textbf{APPNP} and learns  $\gamma_k$s  along with $f_{\theta} (\XMatrix)$,  instead of fixing them to $\alpha^k$ as in APPNP~\cite{gasteiger2018predict}, i.e.:
%Mathematically, GPR-GNN~\cite{chien2020adaptive} can be stated as,
\begin{equation}
\begin{aligned}
&\ZMatrix_{\text{GPR-GNN}} = \sum_{k = 0}^K \gamma_k \HMatrix^ {(k)}, \HMatrix^ {(k)} = \widetilde{\PMatrix} \HMatrix^ {(k-1)},
\end{aligned}
    \label{eq:GPR-GNNSpectralGCN}
\end{equation}
where $\gamma_k$s are learnable parameters.

% Specifically,
% \begin{equation}    
% \YMatrix = \text{softmax}\big( \ZMatrix\big), \ZMatrix_{\text{GPR-GNN}} = \mathcal{P}_{\text{Monomial}}\big(\widetilde{\PMatrix}\big) ~\text{and}~ \HMatrix^ {(0)} = f_{\theta} (\XMatrix)
%     \label{eq:GPR-GNNSpectralGCNWithPoly}
% \end{equation}

\textbf{ChebNet~\cite{defferrard2016convolutional}} replaces the polynomial used in the approximation of Equation~\ref{eq:GPR-GNNSpectralGCN} with 
a Chebyshev polynomial. Since Chebyshev polynomials converge more rapidly to the function that is being (implicitly) approximated~\cite{coskun2016efficient}, this process effectively increases the depth of the propagation.   ChebNet~\cite{defferrard2016convolutional} can be formulated as:
\begin{equation}
\begin{aligned}
\ZMatrix_{\text{ChebNet}} = & \sum_{k = 0}^K \gamma_k \HMatrix^ {(k)}, \quad\\
\HMatrix^ {(k)} = &~2 \widetilde{\PMatrix} 
\HMatrix^ {(k-1)} - \HMatrix^ {(k-2)},
\end{aligned}
    \label{eq:ChebNetSpectralGCN}
\end{equation}

% Specifically,
% \begin{equation}    
% \YMatrix = \text{softmax}\big( \ZMatrix\big), \ZMatrix_{\text{ChebNet}} = \mathcal{P}_{\text{Chebyshev}}\big(\widetilde{\PMatrix}\big) ~\text{and}~ \HMatrix^ {(0)},\HMatrix^ {(1)} = f_{\theta} (\XMatrix)
%     \label{eq:ChebNetSpectralGCNWithPoly}
% \end{equation}
where $\HMatrix^ {(0)}, \HMatrix^ {(1)} =  f_{\theta} (\XMatrix)$.
Similarly, one can replace the polynomial in Equation~\ref{eq:ChebNetSpectralGCN} with any polynomial, provided that the update rules of the chosen polynomial are preserved. In the Spectral GCNs literature, this flexibility allows the creation of various Spectral GCN algorithms. 
Examples include CayleyNet~\cite{levie2018cayleynets}, GPR-GNN~\cite{chien2020adaptive}, BernNet~\cite{he2021bernnet}, and JacobiCon~\cite{wang2022powerful}, respectively utilizing Cayley, monomial, Bernstein, and Jacobi Polynomials.

\section{Methods}
\label{sec:method}
%In this subsection, we present our Arnoldi-based Spectral GCN algorithms, addressing the limitations of existing methods by enabling simultaneous exploration of polynomial samplings and filter functions. 
Spectral GCNs can be more versatile and adaptive
if they are used with explicit filter functions.
For predetermined graph convolution, explicit filter 
functions enable choosing the filter to suit the topology of graph, the nature of the learning task (e.g., homophilic vs. heterophilic graphs), and domain knowledge relating to distribution of labels.
For learnable graph convolution, explicit filter functions enable
initialization of the coefficients to values that suit the learning task. 
The challenge in the application of explicit filter functions is in computing polynomial approximations to these filter functions.
With \algoArnoldi{} (for predetermined graph convolution) and \algoGArnoldi{} (for learnable graph convolution), we address this challenge by using Arnoldi orthonormalization to solve the key challenge associated with ill-conditioned systems that result from the polynomial approximation of filter functions.
The workflow of  \algoArnoldi{} and \algoGArnoldi{} is shown in Figure~\ref{fig:MainFig}.

We begin by defining representative filter functions and polynomial samples, emphasizing that the application of our algorithms extend beyond these specific choices. 
Subsequently, we provide theoretical underpinnings and methodological details of our proposed algorithms. 

% \textbf{Motivation:} To set up our motivation, we start with an example, which is visually shown in Figure~\ref{fig:ExplicitFilter}.   Assume we have two explicit filters, $g_{\text{high\_band\_pass}} \colon (0,2] \rightarrow \mathbb{R}, g_{\text{high\_band\_pass}}  (x) = 1- exp(-10 x^2 )$ and $g_{\text{low\_band\_pass}} \colon (0,2] \rightarrow \mathbb{R}, g_{\text{low\_band\_pass}} (x) = exp(-10 x^2 )$, we want to create two different polynomials, $P(x)$ and $Q(x)$ to represent these filter functions by using 4 Chebyshev samples, $\omega_k = cos \big(\frac{2k -1}{2n} \pi \big)$ for $k = 1, \cdot, \cdot, \cdot, r = 4$. Note that these Chebyshev samples are fixed as $\omega = (0.076,    0.617, 1.382, 1.923]$. We first project these samples on the filter functions one by one as $g_{\text{high\_band\_pass}}  (\omega)$ and $g_{\text{low\_band\_pass}}  (\omega)$, where $g_{\text{high\_band\_pass}}  (\omega) \neq g_{\text{low\_band\_pass}}  (\omega)$. Then, we obtain upper parts (a) and (b) of Figure~\ref{fig:ExplicitFilter}. Finally, our main goal is to obtain two different 3-degree polynomials $P_3(x) \neq Q_3(x)$ that are passing from the projected samples as shown in Figure~\ref{fig:ExplicitFilter} parts (c) and (d), respectively. This way, we connect any filter with any Polynomial (via samples, e.g., Equispaced,  Chebyshev, Legendre, Jacobi, and etc.,)

\subsection{Explicit Filter Functions for Spectral GCNs}

We focus on eight representative filter 
functions, with a view to demonstrating and comparing the utility of explicitly specified filter functions in
the context of spectral GCNs.
%and four polynomial samples to illustrate their application in our proposed Arnoldi-process based Spectral GCN algorithms, namely \algoArnoldi{} and \algoGArnoldi{}. 
These represent four simple and four complex filter functions, establishing propagation rules for homophilic and heterophilic
graphs. 
The simple filters we use correspond to random walks on the graph and operate on the eigen-decomposition of the degree-normalized adjacency matrix ($\PMatrix$).
The complex filters operate on the eigen-decomposition of
the graph Laplacian ($\LMatrix$).
The range of a filter function is therefore determined by 
the bounds on the eigenvalues of the matrix
it operates on.
Specifically, we consider the following filter 
functions, where $0<\alpha<1$ is a hyperparameter:
%representative filter functions are as follows:

{\bf {Simple Filters:}} ($\omega\in (-1, 1)$)

\begin{equ}[!ht]
%\caption{Simple Filters, $\omega\in (-\alpha, \alpha)~ \mathrm{and}%~\beta = 0.1, 
%~\alpha = 0.9$:}
\begin{empheq}{align}%[left = \empheqlbrace]{align}
    \text{Scaled Random Walk:}~ g_0 (\omega) = \frac{1-\alpha}{1-\omega} 
    &
    %\mathrm{for}~\beta {\big( \IMatrix - \alpha \PMatrix\big)}^{-1},  
    \label{eq:g0}
    \\
    \text{Random Walk:}~  g_1 (\omega) = \frac{1}{1-\omega} 
    &
    %\mathrm{for}~ {\big( \IMatrix - \alpha \PMatrix\big)}^{-1}.
    \label{eq:g1}
    \\
     \text{Self-Depressed RW:}~ g_2 (\omega) = \frac{\omega}{1-\omega} 
    &
   %\mathrm{for}~\PMatrix \times {\big( \IMatrix - \alpha \PMatrix\big)}^{-1},
    \label{eq:g2}
    \\
     \text{Neighbor-Depressed RW:}~ g_3 (\omega) = \frac{\omega^2}{1-\omega} 
    &
    %\mathrm{for}~\PMatrix^2 \times {\big( \IMatrix - \alpha \PMatrix\big)}^{-1},
    \label{eq:g3}
\end{empheq}
\end{equ}

{\bf{Complex Filters}:} ($\omega \in (0,2]$)

%\vspace{-0.2in}

%Next, we discuss the computation of the polynomial approximation
%for a given filter function.

%===============================================

\begin{equ}[!ht]
%{\bf Complex Filters, $ \omega \in (0,2]$:}
\begin{empheq}{align}%[left = \empheqlbrace]{align}
     \mathrm{Low}~\mathrm{Pass}:~ &  g_4 (\omega) = e^{-10\omega^2}~~~~~~~~~~~~~~~ 
    \label{eq:g4}
    \\
     \mathrm{High}~\mathrm{Pass}:~ &  
     g_5 (\omega) = 1- e^{-10\omega^2}~~~~~~~~~ 
    \label{eq:g5}
    \\
    \mathrm{Band}~\mathrm{Pass}:~&
     g_6 (\omega) = e^{-10 {(\omega-1)}^2}~~~~~~ 
    \label{eq:g6}
    \\
     \mathrm{Band}~\mathrm{Rejection}:~&
     g_7 (\omega) = 1- e^{-10 {(\omega-1)}^2} 
    \label{eq:g7}
\end{empheq}
\end{equ}
%===============================================
These functions are used to filter the
spectra of the graph Laplacian (Equation~\ref{eq:filter}) to facilitate propagation of the graph signal in accordance with the rules specified by the filter~\cite{matsuhara1974optical}.
However, since it is computationally expensive to compute 
the eigen-decomposition of the Laplacian, we compute and use 
a polynomial approximation to the given filter (Equation~\ref{eq:filterPoly}).
It is important to note that the framework we propose is not limited to explicitly designed filter functions. 
Since the below polynomial approximation scheme is based on sampling the function, it can be used for any filter function that can be sampled in the range of eigenvalues of the graph Laplacian (or the normalized adjacency matrix).

\subsection{Polynomial Approximation of Filter Functions}

For a given filter $g(\omega)$ and integer $K$, 
our objective is to compute a polynomial of  degree
$K$ that best approximates $g(\omega)$:
\begin{equation}
    P(\omega) = \sum_{k = 0}^{K} a_k \omega^k,
    \label{eq:polynomial}
\end{equation}
That is, we seek to compute the coefficients $a_k$ for $k=0$ to $K$ such that $P(\omega)$ provides a good approximation to $g(\omega)$.
We use these coefficients to implement the spectral 
GCN associated with filter $g(\omega)$, using the approximation in Equation~(\ref{eq:filterPoly}).

Well-established techniques in numerical analysis
consider $r$ distinct samples taken from the range of $\omega$, denoted  $\omega_1, \omega_2, \ldots, \omega_r$,
where $r=K$.
%It is important to note that $r$ denotes the number of samples and $K$ denotes the degree of the polynomial, and they are not required to be equal. {\bf ayg: we may want to say some more about this.. in general approximation, r = p-1.. why is it different here?}
In this study, we consider four sampling techniques, namely
equispaced samples (Eq.), Chebyshev sampling (Ch.), Legendre sampling (Le.), and Jacobi sampling (Ja.).
For the range $[l, u]$ of the filter function, the values $\omega_k$ for $1\leq k \leq r$ are sampled from this range by each of these techniques 
as follows:

%\vspace{-0.25in}
\begin{equ}[!ht]
%\caption{Polynomial Samples to approximate the filter functions}
\begin{empheq}{align}%[left = \empheqlbrace]{align}
     \mathrm{Eq.}: & ~\omega_k = l + k\frac{u-l}{r+1}%~~~~~~~~~~~~~~~~~~~~~~~~~~~~~~~~~~~~~~~~~~~~~~ 
    %&: p_{0}%:\mathrm{Equispaced} . 
    \label{eq:p0}
    \\
    \mathrm{Ch.}: &
     ~\omega_k = \frac{u+l}{2} + \frac{u-l}{2}cos \big(\frac{2k -1}{2r} \pi \big) 
    %&: p_{1}%:\mathrm{Chebyshev}.
    \label{eq:p1}
    \\
     \mathrm{Le.}: & \int_{-1}^1 p(\omega)\;\mathrm{d}\omega = \sum_{k=1}^r p(\omega_k)~~~~~~~~~~~~~~~~~~~~~~~~~ 
    %&:p_{2} %:\mathrm{Legendre }
    \label{eq:p2}
    \\
    \mathrm{Ja.}: & \int_{-1}^1 (1+ \omega)p(\omega)\;\mathrm{d}\omega = \sum_{k=1}^r (1+ \omega_k) p(\omega_k) 
    %&:p_{3}%:\mathrm{Jacobi} 
    \label{eq:p3}
\end{empheq}
\end{equ}

%Here, $p_{0}, p_{1}, p_{2}$, and $ p_{3}$ represent \emph{Equispace, Chebyshev, Legendre}, and \emph{Jacobi} polynomial samplings, respectively. For \emph{Simple Filters} and \emph{Complex Filters}, we set $l = -\alpha, u = \alpha $ and $l = 0, u = 2 $, respectively. 
%\vspace{-0.25in}
Here, $p (\omega)$ denotes the polynomial $1 + \omega + \omega^2 + \ldots$.
%denotes the monic polynomial. 
The points for Legendre and Jacobi sampling are computed numerically, using the Gauss Quadrature procedure, since no analytical solution exists for integral-formed polynomials.  ~\cite{golub1969calculation}.
%This prodcedure leverages their orthogonality to $p$ on intervals [-1,1] and (0,1], respectively. 
Once the sampled points in the interval $[-1,1]$ are computed, we scale and shift them to the desired interval $[l,u]$.
%$[-\alpha, \alpha]$ or (0,2],  
%obtained from any polynomial samplings defined in Equations~\ref{eq:p0}-\ref{eq:p3}. 

The next step in computing the approximation involves evaluating the value of the filter functions at these sampled points,  yielding $g (\omega_1), g(\omega_2), \ldots, g(\omega_r)$ to align the polynomials with the filters.
%To approximate these filter functions, computing polynomial coefficients is necessary. 
%TWhile a Vandermonde System represents the function approximation problem, its ill-conditioned nature leads to illegal coefficients (see Appendix). In the following subsection, we propose an alternative inspired by Vandermonde matrix's QR decomposition, employing Arnoldi-Orthogonalization.
Once these values are computed, expressing the right side of Equation~(\ref{eq:polynomial}) in Vandermonde matrix form, we 
obtain:

\vspace{-0.22in}
\begin{equation}
\VMatrix = 
\begin{bmatrix}
1 & \omega_1& {\omega_1}^2 & \dots & {\omega_1}^{K} \\
1 & \omega_2& {\omega_2}^2 & \dots & {\omega_2}^{K}  \\
\dots  & \dots  & \dots  & \dots & \dots  \\
1 & \omega_r& {\omega_r}^2 & \dots & {\omega_r}^{K} 
\end{bmatrix}
\begin{bmatrix}
a_0 \\ a_1 \\ \dots \\ a_K 
\end{bmatrix}
=
\begin{bmatrix}
g (\omega_1) \\ g(\omega_2) \\ \dots \\ g(\omega_r) 
\end{bmatrix}
    \label{eq:Vandermonde}
\end{equation}
A commonly encountered problem in computing polynomial approximations is that the Vandermonde matrix $\VMatrix$ 
is typically ill-conditioned.
%thus direct numerical solution of this sytem leads to illegal coefficients. 

\begin{definition}
The condition number of matrix $A$ is defined as 
$\kappa (A) = \| A \| \cdot \| A^{-1} \|$. 
If $\kappa (A) = O(1)$, the matrix is well-conditioned; otherwise it is ill-conditioned.
%{\bf ayg: Actually, if the condition number is O(1), we generally think of the matrix as well-conditioned. Maybe the ML community has a different interpretation?}
\end{definition}

\begin{theorem}
\label{theorem:appendix}
    Let $\omega_1, \cdot, \cdot, \cdot, \omega_r$  be the samples obtained using one of the sampling techniques in Equation~\ref{eq:p0}-\ref{eq:p3}.  If $\omega_1, \cdot, \cdot, \cdot, \omega_r\in [-\alpha, \alpha]$ with $\alpha \in (0,1)$, then $\kappa (\VMatrix) = \| \VMatrix \| \cdot \| \VMatrix^{\dagger} \| \geq  2^{r-1} {\big(\frac{1}{\alpha}\big)}^r$.
    If $\omega_1, \cdot, \cdot, \cdot, \omega_r\in (0, 2]$, then $\kappa (\VMatrix) = \| \VMatrix \| \cdot \| \VMatrix^{\dagger} \| \geq  2^{r-2}$.
\end{theorem}

The proof of this theorem is provided in the Appendix.
A consequence of this theorem is that direct solution of (\ref{eq:Vandermonde}) leads to inaccurate coefficients.
%Note that use of Lagrange polynomials also leads to an 
%ill-conditioned linear system since Lagrange produces the same coefficients with Vandermonde approach.
For this reason, we solve this system using QR decomposition. % employing Arnoldi Orthonormalization.
Let
%To compute the $a$s, we first compute Vandermonde matrix's QR decomposition as 
$\VMatrix = \QMatrix_{\text{V}} \RMatrix_{V}$ be the QR decomposition of $\VMatrix$ and ${{\QMatrix}^{\dagger}}_{\text{V}}$ denote the pseudo-inverse of $\QMatrix_{\text{V}}$. A solution for Equation~(\ref{eq:Vandermonde}) can be obtained as:
\begin{equation}
a = ({{{\QMatrix}^{\dagger}}_{\text{V}}})g
    \label{eq:weights}
\end{equation}

%However, this approach does not work well in practice since Vandermonde matrix is ill-conditioned as shown the below theorem, and as a result, the found coefficients are illegal to approximate a given filter, see Appendix for further details.

% \begin{algorithm}[!t]
% \caption{Calculate coefficients with Arnoldi/Lanczos orhonormalization on $\OmegaMatrix$}
% \begin{algorithmic}[1]
% \Require $l, u$: lower and upper bounds, function $g$, polynomial $p$, number of samples $r$, and subspace dimension $K$.  
% \State Sample $r$ samples in interval $[l,u]$ by polynomial, $p$, i.e., $\omega_1, \cdot, \cdot, \cdot, \omega_r$.
% \State Map the samples onto the given function $g$, i.e., obtain $g(\omega_1), \cdot, \cdot, \cdot, g(\omega_r)$
% \State Create diagonal matrix $\OmegaMatrix$ as $\mathrm{diag} (\omega_1, \cdot, \cdot, \cdot, \omega_r)$.
% \State $\boldsymbol{Q} = \boldsymbol{1}$ \Comment{1 vector}
% \For {m = 1 to K}
% \State q = $\OmegaMatrix$ $\boldsymbol{Q}(:, m)$
% \For{l = 1 to m}
% \State $\boldsymbol{H}(l,m) = \boldsymbol{Q}(:,l).*\frac{q}{\sqrt{r}}$
% \State $q = q- \boldsymbol{H}(l,m)* \boldsymbol{Q}(:,l)$
% \EndFor
% \State $\boldsymbol{H}(l+1,l) = \frac{\norm{q}}{\sqrt{r}}$
% \State $\boldsymbol{Q} = [\boldsymbol{Q}~~\frac{q}{\boldsymbol{H}(l+1,l)}]$ \Comment{Append columns}
% \EndFor
% \State $a_{A} = \boldsymbol{Q}^{\dagger} g(\omega)$

% \label{algo:Arnoldi}
% \end{algorithmic}
% \end{algorithm}

\begin{algorithm}[t]
   \caption{Compute coefficients with Arnoldi/Lanczos orhonormalization on $\OmegaMatrix$}
   \label{algo:Arnoldi}
\begin{algorithmic}
   \STATE {\bfseries Input:} $l, u$: lower and upper bounds, function $g$, sampling $p$, number of samples $r$, and degree $K$.
   \STATE Sample $r$ samples, i.e., $\omega_1, \cdot, \cdot, \cdot, \omega_r$.
   \STATE Obtain $g(\omega_1), \cdot, \cdot, \cdot, g(\omega_r)$
   \STATE Create diagonal matrix $\OmegaMatrix$ as $\mathrm{diag} (\omega_1, \cdot, \cdot, \cdot, \omega_r)$.
   \STATE $\boldsymbol{Q} = \boldsymbol{1}$
    \FOR{$m=1$ {\bfseries to} $K$}
        \STATE q = $\OmegaMatrix$ $\boldsymbol{Q}(:, m)$
        \FOR{$L=1$ {\bfseries to} $m$}
            \STATE $\boldsymbol{H}(l,m) = \boldsymbol{Q}(:,l).*\frac{q}{\sqrt{r}}$
            \STATE $q = q- \boldsymbol{H}(l,m)* \boldsymbol{Q}(:,l)$
        \ENDFOR
        \STATE $\boldsymbol{H}(l+1,l) = \frac{\norm{q}}{\sqrt{r}}$
        \STATE $\boldsymbol{Q} = [\boldsymbol{Q}~~\frac{q}{\boldsymbol{H}(l+1,l)}]$
    \ENDFOR
    \STATE $a_{A} = \boldsymbol{Q}^{\dagger} g(\omega)$
\end{algorithmic}
\end{algorithm}

%------------------------------------------------
% Specifically our Spectral GCN is follows:
% \begin{equation}    
% \YMatrix = \text{softmax}\big( \ZMatrix\big), \ZMatrix_{\text{Ours}} = \mathcal{P}_{\text{AnyPolynomial}}\big(\AMatrix_{\text{sym}}\big) = \HMatrix^ {(0)} + \xi_1 \AMatrix_{\text{sym}} \HMatrix^ {(1)} +  \xi_2 {\AMatrix_{\text{sym}}}^2 \HMatrix^ {(2)} + ,\cdot, \cdot, \cdot,~\text{and}~ \HMatrix^ {(0)} = f_{\theta} (\OmegaMatrix)
%     \label{eq:GPR-GNNSpectralGCNWithPoly}
% \end{equation}

Note that using QR decomposition does not mitigate the
ill-conditioning problem by itself.
% To mitigate the ill-conditioning of the system, 
 %compute the coefficients $a_k$ for , 
 To address this, we use an alternate QR-decomposition  using $\OmegaMatrix = \text{diag} (\omega_1, \cdot, \cdot, \cdot, \omega_r)$.
 In other words, 
 let $\QMatrix_A$ denote the orthonormal basis that is computed using Arnoldi Orthogonalization on $\OmegaMatrix$. We compute the coefficients of the polynomial approximation as:
 \begin{equation}
a = ({\QMatrix_{A}}^{\dagger }) g
\label{eq:arnoldisol}
\end{equation}
 and show that this formulation leads to a correct solution for the linear system (Theorem 2) and produces accurate coefficients for our polynomial approximation (Theorem 3).
 %To compute this alternate decomposition, we use Arnoldi orthonormalization, as shown in Algorithm 1.
 %that employs a matrix instead of matrix power for QR decomposition.
 %mitigating the ill-conditioning effect. The following theorem asserts that we can compute an alternative orthonormal basis for~$\QMatrix_{\VMatrix}$.

%\begin{theorem}
%\label{thm:theorem1}
  % An alternative orthonormal $\QMatrix_A$ basis can be computed by using Arnoldi Orthogonalization produce for computing the polnomial coefficients as $a_{A} = ({\QMatrix_{A}}^{\dagger }) g$. 
%Let $\OmegaMatrix = \text{diag} (\omega_1, \cdot, \cdot, \cdot, \omega_r)$. Let $\QMatrix_A$ denote the alternate orthonormal basis that is computed using Arnoldi Orthogonalization on $\OmegaMatrix$. Then the polynomial coefficients can be computed as 
%$a = ({\QMatrix_{A}}^{\dagger }) g$. 
%\end{theorem}

%The proof of this theorem is given in the Appendix
%and 
The procedure for computing $\QMatrix_{A}$ is given 
in Algorithm~\ref{algo:Arnoldi}. 
Note that, since the $\omega_k$s are real, Arnoldi orthonormalization here can be implemented
using Lanczos' algorithm.
%where we show that we can generate an alternative orthonormal $\QMatrix_{A}$ basis via well-known Alnordi algorithm and procedure is summarized in Algorirthm~\ref{algo:Arnoldi}. 
Specifically, we compute orthonormal matrix $\QMatrix_A$, tridiagonal $\TMatrix$, and an almost-zero matrix $\AlmostZero$ to satisfy:
\begin{equation}
    \OmegaMatrix \QMatrix_A = \QMatrix_A\TMatrix + \AlmostZero
    \label{eq:compactQAwithLanczos2} 
\end{equation}
where $\OmegaMatrix \in \mathbb{R}^{r \times r}$, $\QMatrix_A \in \mathbb{R}^{(K+1) \times r}$, $\TMatrix \in \mathbb{R}^{(K+1) \times (K+1)}$, and $\AlmostZero \in \mathbb{R}^{(K+1) \times r}$ is all zero except its 
last column.

The next theorem guarantees that  Arnoldi/Lanczos Orthonormalization can be used to obtain an alternative QR decomposition to original Vandermonde matrix.

\begin{theorem}
\label{thm:theorem2}
    Given $\OmegaMatrix$, let orthonormal $\QMatrix$, tridiagonal $\TMatrix$ and almost-zero $\AlmostZero$ matrices be computed using Lanczos algorithm to satisfy Equation~(\ref{eq:compactQAwithLanczos2}). Then we can obtain a QR-decomposition for the Vandermonde matrix $\VMatrix$ as $\VMatrix^{(*)} = \QMatrix \RMatrix$, such that $\VMatrix^{(*)} = \VMatrix/\norm{\identityvec}$ and $\RMatrix = [\identityvec_{1}, \TMatrix \identityvec_{1}, \cdot, \cdot, \cdot, \TMatrix^{K} \identityvec_{1}]$, where $\identityvec_{1}$ denotes the $K+1$-dimensional vector of ones.     
\end{theorem}

The proof of this theorem is provided in the Appendix.
This theorem establishes that it is possible to compute an alternate orthonormal basis for the Vandermonde matrix
using the Arnoldi/ Lanczos process, which can be used 
to solve the linear system of Equation (\ref{eq:Vandermonde}) using Equation~(\ref{eq:arnoldisol}). 
However, ensuring the accuracy of these coefficients requires addressing whether $\QMatrix_A$ leads 
to precise computation of coefficients. 
Theorem~\ref{theorem:main} provides the basis for accuracy of $a_{A} = ({\QMatrix}^{\dagger }) g$. Specifically, we establish that the condition number of $(\QMatrix^{\dagger} \QMatrix)$ is close to one, ensuring the accuracy of the solution.

\begin{theorem}
\label{theorem:main}
    Let $\omega_1, \cdot, \cdot, \cdot, \omega_r$ be points sampled from interval $(-1, 1)~ {\mathrm{or}}\ (0,2]$ using one of the techniques shown in Equation~\ref{eq:p0}-\ref{eq:p3}. Let $\OmegaMatrix = \text{diag} (\omega_1, \cdot, \cdot, \cdot, \omega_r)$, $\QMatrix$ be the orthonormal basis obtained by applying Arnoldi/ Lanczos orthonormalization on $\OmegaMatrix$, and ${\QMatrix^{\dagger}}$ be the pseudo-inverse of  ${\QMatrix}$. 
    If the Krylov subspace used for the orthogonalization has $K = r$ dimensions,
    %the following equation holds:
    %\[{\big ( \sum_{i= K}^r {\sigma_i}^2 \big)}^{1/2} \leq {\Big( \sum_{i=K}^r i-1\Big)}^{1/2}\],
    %and thus 
    then $\kappa (\QMatrix^{\dagger} \QMatrix) \approx 1.01$.
\end{theorem}

In summary, these three theorems show that computing the QR decomposition of $\VMatrix$ with Arnoldi/ Lanczos process on $\OmegaMatrix$ enables us to produce accurate polynomial coefficients, since condition number of $\QMatrix$ is bounded, while performing QR decomposition directly on $\VMatrix$ produces inaccurate coefficients due to the ill-conditioned nature of the matrix $\VMatrix$. 
%After finding these coefficients, $a_k$s, we propose two Spectral GCN algorithms:

% \begin{table*}[ht!]
%     \centering
%        \caption{Statistics of the datasets.}
%        \vspace{0.1in}
%     \resizebox{.75\textwidth}{!}{
%     \begin{tabular}{zxxxxxyyyyy}
%     \hline\hline
%         {\bf Statistics } & {\bf Cora} & {\bf Citeseer} & {\bf Pubmed} & {\bf Photos} & {\bf Computers} & {\bf Texas} & {\bf Cornell} & {\bf Actor} & {\bf Cham.} & {\bf Squir.}  \\
%         \hline
%         Nodes & 2708  &3327 &  19717 & 7650 & 13752& 183  &183 &  7600 & 2277 & 5201\\
        
%         Edges & 5278  & 4552 &44324 &  119081 & 245861& 279  &277 &  26659 & 31371 & 198353\\
        
%         Features & 1433 &3703 & 500 & 745 & 767& 1703  &1703 &  932 & 2325 & 2089 \\
        
%         Classes & 7 & 6& 5 & 8 & 10& 5  &5 &  5 & 5 & 5 \\
%         Node Homophily & 0.82 & 0.72& 0.79 & 0.84 & 0.79& 0.10  &0.11 &  0.21 & 0.25 & 0.22 \\
%         Edge Homophily & 0.81 & 0.74& 0.80 & 0.83 & 0.78& 0.11 &0.13 &  0.22 & 0.24 & 0.22 \\
%         Adjusted Homophily & 0.77 & 0.67& 0.68 & 0.78& 0.68& -0.26  &-0.21 &  0.0044 & 0.0331 & 0.0070\\
%         Node Label Informativeness & 0.61 & 0.45& 0.40 & 0.72 & 0.63& 0.24  &0.19 &  0.0023 & 0.0567 & 0.0026\\
%         Edge Label Informativeness & 0.59 & 0.45& 0.41 & 0.67 & 0.53& 0.16  &0.16 &  0.0002 & 0.0479 & 0.0014 \\
        
%     \hline\hline
%     \end{tabular}}
    
%     \vspace{-0.1in}
 
%     \label{tab:TableStatic}
% \end{table*}

\begin{table*}[t]
    \centering
       \caption{{\bf Datasets used in our experiments.}
       For each dataset, basic statistics and the characteristics of graph topology in relation to the distribution of classes are shown.} 
       \vspace{0.05in}
    \resizebox{.95\textwidth}{!}{
    \begin{tabular}{zxxxxyy}
    \hline\hline
        {\bf Datasets } & {\bf \# of Nodes} & {\bf \# of Edges} & {\bf \# of Features} & {\bf \# of Classes}  & {\bf Adj-Homophily} & {\bf Label Informativeness}  \\
        \hline
        \multicolumn{7}{l}{{\em Homophilic Graphs}}\\
        \hline
        {\bf Cora} & 2708  & 5278  &1433 & 7   &0.77 & 0.61 \\      
        {\bf Citeseer} & 3327 & 4552 &3703 &  6 & 0.67 &  0.45  \\
        {\bf Pubmed} & 19717 &44324& 500 & 5 & 0.68 &  0.40 \\
        {\bf Photo} & 7650 & 119081& 745 & 8 & 0.78 & 0.67 \\
        {\bf Computer}& 13752 & 245861& 767 &10 & 0.68 &  0.63 \\
         \hline
         \multicolumn{7}{l}{{\em Small Heterophilic Graphs}}\\
         \hline
        {\bf Texas} & 183 & 279& 1703 & 5 & -0.26 &  0.24 \\
        {\bf Cornell} & 183 & 277& 1703 & 5&-0.21 &  0.19 \\
        {\bf Actor} &   7600 & 26659& 932 & 5 &0.0044  &  0.0023 \\
       {\bf Chameleon} & 2277 & 31371& 2325 & 5 & 0.0331 &  0.0567 \\
       {\bf Squirrel} & 5201 & 198353& 2089 &5 & 0.0070 &  0.0026  \\
       \hline
         \multicolumn{7}{l}{{\em Large Heterophilic Graphs}}\\
       \hline
       {\bf Roman-Empire} &22662 &32927 & 300&18& -0.05 &0.11 \\
       {\bf Amazon-Ratings} &24492 & 93050 & 300&5 &0.13 &0.04 \\
       {\bf Minesweeper} &10000 & 39402 & 7&2& 0.009 &0.000  \\
       {\bf Tolokers} & 11758 & 519000 & 10&2 & 0.0925 &0.0178  \\
       {\bf Questions} & 48921 & 153540 & 301&2 & 0.0206 &0.0068 \\        
    \hline\hline
    \end{tabular}}
    
    \vspace{-0.1in}
 
    \label{tab:TableStats}
\end{table*}

\subsection{Generalized Spectral GCNs}

The framework described above enables
computation of accurate polynomial approximations to any explicit filter function.
Using this framework, we propose two Spectral GCNs that operate with an explicitly formulated filter function:
(i) \textbf{Arnoldi-GCN} for predetermined graph convolution; and (ii)  \textbf{G-Arnoldi-GCN} for learnable graph convolution.
These spectral GCNs are {\em generalized} in the sense 
that they can work with any filter function.
Furthermore, any spectral GCN in the literature can be formulated within this framework.

Given a filter function $g(\omega): D \rightarrow \mathbb{R}$ and integer $K$, 
let $a_0$, $a_1$, ..., $a_K$ denote the coefficients
of the polynomial that approximates $g(\omega)$, computed
using the procedure described in the previous section.
If $g$ operates on the degree-normalized adjacency 
matrix, then $D=(-1,1)$.
If $g$ operates on the graph Laplacian, then $D=(0,2]$.
Using this filter function, \textbf{Arnoldi-GCN} and \textbf{G-Arnoldi-GCN} are implemented as follows:

\textbf{Arnoldi-GCN}
\begin{equation}
\begin{aligned}
\YMatrix = &~\text{softmax}\big( \ZMatrix_{\text{ARNOLDI}}\big), \quad\\
\ZMatrix_{\text{ARNOLDI}} = &\sum_{k = 0}^K a_{k} \HMatrix^ {(k)}, \HMatrix^ {(0)} = f_{\theta} (\XMatrix), \quad\\
\HMatrix^ {(k)} = & \begin{cases} 
      \widetilde{\PMatrix}  \HMatrix^ {(k-1)}, &~~~D= (-1,1) \\
      \widetilde{\LMatrix}  \HMatrix^ {(k-1)}, &~~~D= (0,2]
   \end{cases}   
\end{aligned}
    \label{eq:ArnoldiSpectralGCN}
\end{equation}
Here, $\widetilde{\PMatrix}$ and $\widetilde{\LMatrix}$ are self-loop added versions of original normalized adjacency and Laplacian matrices, respectively.
$f_{\theta}$ denotes the neural network with learnable parameters $\theta$.

\textbf{G-Arnoldi-GCN}
\begin{equation}
\begin{aligned}
\YMatrix = &~\text{softmax}\big( \ZMatrix_{\text{G-ARNOLDI}}\big), \quad\\
\ZMatrix_{\text{G-ARNOLDI}} = &\sum_{k = 0}^K \gamma_{k} \HMatrix^ {(k)}, \HMatrix^ {(0)} = f_{\theta} (\XMatrix), \quad\\
% \HMatrix^ {(k)} = & \begin{cases} 
%       \widetilde{\PMatrix}  \HMatrix^ {(k-1)}, &~~~ \omega \in (-1,1) \\
%       \widetilde{\LMatrix}  \HMatrix^ {(k-1)}, &~~~ \omega \in (0,2]
%    \end{cases}   
\end{aligned}
    \label{eq:GArnoldiSpectralGCN}
\end{equation}
where $\gamma_k$ are learneable parameters that are optimized alongside $\theta$ in an end-to-end fashion by initializing $\gamma_k = a_k$.
%and $\HMatrix^ {(k)}$s are defined as in Equation~\ref{eq:ArnoldiSpectralGCN}.

%% file: Results.tex
\section{Results}
\label{eq:results}

\begin{table}[ht!]
 \caption{Accuracy of node classification for \algoArnoldi{}, \algoGArnoldi{}, and state-of-the-art algorithms. Mean and standard deviation of five runs are reported for all experiments.}
    \label{tab:SSHomo}
 
 \vspace{-0.05in}
 \centerline{{\bf Semi-supervised} on {\bf Homophilic} datasets}  %\vspace{-0.05in}
\resizebox{\columnwidth}{!}{%
    \centering
    \begin{tabular}{zxxxxx}
    \hline\hline
        {\bf Method} & {\bf Cora} & {\bf Citeseer} & {\bf Pubmed} & {\bf Photos} & {\bf Computers} \\
        \hline
        GCN & $ {75.01}_{\pm{{2.19}}}$  & $ {67.57}_{\pm{{1.33}}}$ & $ {84.17}_{\pm{{0.33}}}$ & ${81.81}_{\pm 2.63}$ & ${68.58}_{\pm { {2.49}}}$\\
        
        GAT& ${77.22}_{\pm { {1.99}}}$ & ${66.42}_{\pm { {1.19}}}$ & ${83.32}_{\pm { {0.47}}}$ & ${86.66}_{\pm 1.34}$ & ${72.38}_{\pm { {2.67}}}$\\
        
        APPNP& ${79.93}_{\pm { {1.69}}}$  & ${68.27}_{\pm { {1.27}}}$ & ${84.22}_{\pm { {0.19}}}$ & ${83.24}_{\pm 2.97}$ & ${67.46}_{\pm { {2.43}}}$\\
        
        ChebNet & ${69.58}_{\pm { {3.61}}}$  & ${65.36}_{\pm { {3.01}}}$ & ${83.88}_{\pm { {0.54}}}$ & ${88.00}_{\pm 1.59}$ & ${79.25}_{\pm { {0.30}}}$\\

        JKNet & ${71.31}_{\pm { {2.65}}}$  & ${61.36}_{\pm { {3.94}}} $ & ${82.92}_{\pm { {0.56}}}$ & ${78.25}_{\pm 9.29}$& ${66.43}_{\pm { {5.69}}}$\\
        
        GPR-GNN & ${79.65}_{\pm { {1.62}}}$ & ${66.92}_{\pm { {1.58}}}$ & ${84.21}_{\pm { {0.54}}}$ & ${88.55}_{\pm 1.29}$ & ${80.73}_{\pm { {1.49}}}$\\

        BernNet& ${73.39}_{\pm { {2.78}}}$ & ${65.84}_{\pm { {1.54}}}$ & $ {84.20}_{\pm { {0.71}}}$ & ${86.33}_{\pm 1.02}$ & ${79.25}_{\pm { {1.28}}}$\\
        JacobiConv& ${80.02}_{\pm { {1.05}}}$ & ${68.23}_{\pm { {1.32}}}$ & $ {84.32}_{\pm { {0.65}}}$ & ${86.41}_{\pm 1.05}$ & ${81.54}_{\pm { {1.12}}}$\\
        
        \hline\hline
               \algoArnoldi{} & $ {{ 80.25}_{\pm{0.43}}} $ &$ { {67.81}_ {\pm {0.39}}}$  & $ { {84.02}_ {\pm {0.29}}}$ & $ {88.25}_{\pm 0.43} $ & $ {78.81}_{\pm 1.13} $ \\
               \algoGArnoldi{} & $\bf {{ 82.33}_{\pm{0.35}}} $   &$ { \bf {69.88}_ {\pm {0.48}}}$    & $ { \bf {85.23}_ {\pm {0.25}}}$  & $\bf {92.46}_{\pm 0.55} $ & $\bf {83.81}_{\pm 1.07} $ \\
    \hline\hline
    \end{tabular}%
    }
%    \vspace{-0.1in}

%\end{table}

%\begin{table}[ht!]
%    \caption{Accuracy of {\bf semi-supervised} node classification on {\bf Small-Heterophilic} datasets.}
    \vspace{0.05in}
     \centerline{{\bf Semi-supervised} on {\bf Small Heterophilic} datasets}  
\resizebox{\columnwidth}{!}{%
    \centering
    \begin{tabular}{zyyyyy}
    \hline\hline
        {\bf Method} & {\bf Texas} & {\bf Cornell} & {\bf Actor} & {\bf Chameleon} & {\bf Squirrel} \\
        \hline
        GCN& ${32.13}_{\pm { {15.90}}}$  &${22.08}_{\pm { {13.43}}}$ & ${22.45}_{\pm { {1.09}}}$&  $ {{39.89}_{\pm{2.49}}}$ & ${{29.66}_{\pm{3.52}}}$ \\
        
        GAT& ${34.27}_{\pm { {20.06}}}$ & ${24.39}_{\pm{ {12.58}}}$ & ${24.31}_{\pm { {2.17}}}$ & ${ 37.86}_{\pm { {3.24}}}$ & ${24.56}_{\pm { {2.79}}}$ \\
        
        APPNP& ${34.67}_{\pm { {14.53}}}$   & ${34.98}_{\pm { {19.61}}}$ & ${28.41}_{\pm { {3.12}}}$ & ${29.38}_{\pm { {1.60}}}$ & ${21.11}_{\pm { {1.53}}}$ \\
        
        ChebNet& ${32.13}_{\pm { {13.73}}}$  &  $ {27.57}_{\pm { {9.71}}}$& ${22.00}_{\pm { {3.58}}}$ & ${36.41}_{\pm { {2.07}}}$& ${26.43}_{\pm { {1.33}}}$ \\

        JKNet& ${30.75}_{\pm { {19.43}}}$  & ${25.20}_{\pm { {21.84}}}$& ${21.02}_{\pm { {2.44}}}$ & ${32.66}_{\pm { {4.74}}}$& ${24.20}_{\pm { {3.01}}}$\\
        
        GPR-GNN& ${33.56}_{\pm { {13.49}}}$ & ${38.84}_{\pm { {21.45}}}$ & ${27.70}_{\pm { {1.80}}}$ & ${33.23}_{\pm { {5.80}}}$ & ${23.43}_{\pm { {2.30}}}$ \\

        BernNet& ${40.69}_{\pm { {19.26}}}$ & ${39.32}_{\pm { {13.95}}}$ & $\bf{28.85}_{\pm { {1.48}}}$ & ${34.73}_{\pm { {2.45}}}$ & ${22.38}_{\pm { {1.43}}}$ \\
        JacobiConv& ${41.23}_{\pm { {6.32}}}$ & ${39.23}_{\pm { {6.32}}}$ & $ {26.37}_{\pm { {1.71}}}$ & ${41.12}_{\pm 2.31}$ & ${\bf{32.23}_{\pm { {1.06}}}}$\\
        \hline\hline
               \algoArnoldi{} & $ {{ 63.20}_{\pm{2.67}}} $  & $ {{ 51.24}_{\pm{1.56}}} $  &  $ {{26.63}_{\pm{1.12}}} $ & $ {{40.25}_{\pm{2.15}}} $   & $ {{24.12}_{\pm{1.94}}} $  \\
               \algoGArnoldi{} & $\bf {{ 66.20}_{\pm{3.21}}} $  & $\bf {{ 55.87}_{\pm{0.73}}} $   &  $ {{27.73}_{\pm{1.81}}} $ & $ \bf{{43.35}_{\pm{3.35}}} $   & $ {{26.64}_{\pm{2.01}}} $  \\
    \hline\hline
    \end{tabular}%
    }
    %\vspace{-0.1in}
    %\label{tab:SSHeto}
%\end{table}

%--------------------------------------------------

%--------------------------------------------------

%\begin{table}[ht!]
%    \caption{Accuracy of {\bf semi-supervised} node classification on {\bf Large-Heterophilic} datasets.}
 %      \vspace{-0.1in}
  \vspace{0.05in}
     \centerline{{\bf Semi-supervised} on {\bf Large Heterophilic} datasets}  
\resizebox{\columnwidth}{!}{%
    \centering
    \begin{tabular}{zyyyyy}
    \hline\hline
       {\bf Method} & {\bf roman-empire} & {\bf amazon-ratings} & {\bf minesweeper} & {\bf tolokers} & {\bf questions} \\
        \hline
        GCN& ${29.02}_{\pm { {0.70}}}$  &${28.97}_{\pm { {1.42}}}$ & ${72.60}_{\pm { {1.08}}}$&  $ {{73.11}_{\pm{3.61}}}$ & $ {{59.86}_{\pm{2.42}}}$ \\
        
        GAT& ${37.26}_{\pm { {0.75}}}$ & ${29.97}_{\pm{ {2.33}}}$ & ${73.51}_{\pm { {0.99}}}$ & ${ 71.11}_{\pm { {0.67}}}$ & ${64.39}_{\pm { {0.74}}}$ \\
        
        APPNP& ${35.30}_{\pm { {0.95}}}$   & ${29.88}_{\pm { {0.53}}}$ & ${67.75}_{\pm { {1.12}}}$ & ${68.99}_{\pm { {0.48}}}$ & ${46.80}_{\pm { {0.61}}}$ \\
        
        ChebNet& ${35.97}_{\pm { {1.33}}}$  &  $ {27.76}_{\pm { {1.21}}}$& ${73.18}_{\pm { {0.32}}}$ & ${70.53}_{\pm { {1.25}}}$& ${64.51}_{\pm { {0.34}}}$ \\

        JKNet& ${28.81}_{\pm { {0.33}}}$  & ${30.50}_{\pm { {1.25}}}$& ${72.48}_{\pm { {1.05}}}$ & $\bf{73.42}_{\pm { {1.44}}}$& ${56.55}_{\pm { {1.25}}}$\\
        
        GPR-GNN& ${36.13}_{\pm { {1.24}}}$ & ${30.03}_{\pm { {1.81}}}$ & ${76.87}_{\pm { {0.80}}}$ & ${68.64}_{\pm { {.87}}}$ & ${54.13}_{\pm { {3.84}}}$ \\

        BernNet& ${39.63}_{\pm { {0.67}}}$ & ${29.32}_{\pm { {1.05}}}$ & ${75.49}_{\pm { {1.10}}}$ & ${69.08}_{\pm { {0.93}}}$ & ${56.27}_{\pm { {0.83}}}$ \\
        JacobiConv& ${41.02}_{\pm { {0.34}}}$ & ${30.24}_{\pm { {1.20}}}$ & ${74.13}_{\pm { {1.52}}}$ & ${65.15}_{\pm { {0.73}}}$ & ${56.21}_{\pm { {1.03}}}$ \\
        \hline\hline
             \algoArnoldi{} & $ {{ 53.04}_{\pm{0.34}}} $  & $ {{ 43.71}_{\pm{0.61}}} $  &  $ {{69.733}_{\pm{1.03}}} $ & $ {{70.25}_{\pm{2.15}}} $   & $ {{56.12}_{\pm{0.94}}} $  \\
               \algoGArnoldi{} & $\bf {{ 69.63}_{\pm{0.37}}} $  & $\bf {{ 48.68}_{\pm{0.70}}} $   &  $ \bf{{88.52}_{\pm{0.73}}} $ & $ {{73.01}_{\pm{1.41}}} $   & $ \bf{{66.18}_{\pm{1.22}}} $  \\
    \hline\hline
    \end{tabular}%
    }
    %\vspace{0.1in}
%    \label{tab:SSHetoLarge}
%\end{table}

%\begin{table}[t!]
%    \caption{Accuracy of {\bf full-supervised} node classification on {\bf Homophilic} datasets.}
\end{table}

\begin{table}[ht!]
 \caption{Accuracy of node classification for \algoArnoldi{}, \algoGArnoldi{}, and state-of-the-art algorithms. Mean and standard deviation of five runs are reported for all experiments.}
    \label{tab:FullHomo}
  \vspace{0.05in}
     \centerline{{\bf Fully-supervised} on {\bf Homophilic} datasets}  
\resizebox{\columnwidth}{!}{%
    \centering
    \begin{tabular}{zxxxxx}
    \hline\hline
        {\bf Method} & {\bf Cora} & {\bf Citeseer} & {\bf Pubmed} & {\bf Photos} & {\bf Computers} \\
        \hline
        GCN& $ {86.45}_{\pm{{1.02}}}$  & $ {79.67}_{\pm{{1.06}}}$ & $ {87.12}_{\pm{{0.49}}}$ & ${86.74}_{\pm 0.62}$ & ${83.21}_{\pm { {0.29}}}$\\
        
        GAT & ${87.11}_{\pm { {1.52}}}$ & ${79.63}_{\pm { {1.37}}}$ & ${86.69}_{\pm { {0.70}}}$ & ${90.72}_{\pm 2.33}$ & ${83.15}_{\pm { {0.39}}}$\\
        
        APPNP & ${87.78}_{\pm { {1.78}}}$  & ${79.50}_{\pm { {1.33}}}$ & ${88.31}_{\pm { {0.39}}}$ & ${85.15}_{\pm 0.92}$ & ${82.51}_{\pm { {0.31}}}$\\
        
        ChebNet& ${86.20}_{\pm { {0.72}}}$  & ${78.11}_{\pm { {1.19}}}$ & ${86.03}_{\pm { {0.56}}}$ & ${ 92.48}_{\pm 0.39}$ & ${85.34}_{\pm { {0.60}}}$\\

        JKNet & ${85.87}_{\pm { {1.56}}}$  & ${76.56}_{\pm { {1.37}}} $ & ${85.52}_{\pm { {0.46}}}$ & ${90.04}_{\pm 1.38}$& ${83.18}_{\pm { {1.21}}}$\\
        
        GPR-GNN & ${87.71}_{\pm { {0.85}}}$ & ${79.57}_{\pm { {1.47}}}$ & ${87.27}_{\pm { {0.48}}}$ & ${93.28}_{\pm 0.62}$ & ${84.68}_{\pm { {0.60}}}$\\

        BernNet& ${87.37}_{\pm { {1.09}}}$ & ${79.41}_{\pm { {1.73}}}$ & $ {88.84}_{\pm { {0.63}}}$ & ${92.38}_{\pm 0.58}$ & ${86.24}_{\pm { {0.31}}}$\\
        JacobiConv& ${88.23}_{\pm { {0.42}}}$ & ${79.32}_{\pm { {0.92}}}$ & $ {88.02}_{\pm { {0.56}}}$ & ${93.10}_{\pm 0.61}$ & ${86.43}_{\pm { {0.52}}}$\\
        \hline\hline
               \algoArnoldi{} & $ {{ 89.23}_{\pm{0.33}}} $   &$ {  {79.21}_ {\pm {1.02}}}$    & $ { {87.65}_ {\pm {0.54}}}$  & $ {92.06}_{\pm 0.45} $ & $ {86.24}_{\pm 0.42} $ \\
               \algoGArnoldi{} & $\bf {{ 90.64}_{\pm{0.53}}} $   &$ { \bf {81.71}_ {\pm {0.87}}}$    & $ { \bf {89.25}_ {\pm {0.43}}}$  & $\bf {94.96}_{\pm 0.54} $ & $\bf {87.66}_{\pm 0.28} $ \\
    \hline\hline
    \end{tabular}%
    }
    %\vspace{0.1in}

 %   \label{tab:FullHomo}
%\end{table}
%---------------------------------------------------------
%\begin{table}[t!]
%    \caption{Accuracy of {\bf full-supervised} node classification on {\bf Small-Heterophilic} datasets.}
  \vspace{0.05in}
     \centerline{{\bf Fully-supervised} on {\bf Small Heterophilic} datasets}  
\resizebox{\columnwidth}{!}{%
    \centering
    \begin{tabular}{zyyyyy}
    \hline\hline
        {\bf Method} & {\bf Texas} & {\bf Cornell} & {\bf Actor} & {\bf Chameleon} & {\bf Squirrel} \\
        \hline
        GCN& ${75.10}_{\pm { {1.20}}}$  &${66.94}_{\pm { {1.41}}}$ & ${33.26}_{\pm { {1.24}}}$&  $ {{60.96}_{\pm{0.68}}}$ & $ {{45.76}_{\pm{0.52}}}$ \\
        
        GAT& ${78.64}_{\pm { {0.96}}}$ & ${76.27}_{\pm{ {1.32}}}$ & ${35.90}_{\pm { {0.17}}}$ & ${ 63.40}_{\pm { {0.49}}}$ & ${42.98}_{\pm { {0.65}}}$ \\
        
        APPNP & ${ 81.10}_{\pm { {1.30}}}$   & ${91.86}_{\pm { {0.51}}}$ & ${38.68}_{\pm { {0.46}}}$ & ${52.51}_{\pm { {1.30}}}$ & ${35.27}_{\pm { {0.49}}}$ \\
        
        ChebNet& ${86.23}_{\pm { {0.93}}}$  &  $ {85.46}_{\pm { {1.09}}}$& ${37.80}_{\pm { {0.65}}}$ & ${59.94}_{\pm { {1.27}}}$& ${40.81}_{\pm { {0.37}}}$ \\

        JKNet & ${75.51}_{\pm { {1.43}}}$  & ${67.24}_{\pm { {1.65}}}$& ${33.27}_{\pm { {1.04}}}$ & ${62.78}_{\pm { {1.34}}}$& ${44.82}_{\pm { {0.54}}}$\\
        
        GPR-GNN& ${82.22}_{\pm { {0.74}}}$ & ${91.45}_{\pm { {1.62}}}$ & ${39.48}_{\pm { {0.83}}}$ & ${66.91}_{\pm { {1.21}}}$ & ${49.80}_{\pm { {0.56}}}$ \\

        BernNet& ${91.77}_{\pm { {1.06}}}$ & ${92.02}_{\pm { {1.70}}}$ & ${41.60}_{\pm { {1.10}}}$ & ${67.76}_{\pm { {1.25}}}$ & ${50.61}_{\pm { {0.72}}}$ \\
        JacobiConv& ${91.23}_{\pm { {1.43}}}$ & ${90.17}_{\pm { {1.61}}}$ & ${40.23}_{\pm { {1.05}}}$ & $\bf{{73.12}_{\pm { {1.16}}}}$ & $\bf{56.20}_{\pm { {1.95}}}$ \\
        \hline\hline
               \algoArnoldi{} & ${{ 92.03}_{\pm{1.24}}} $  & ${{ 91.92}_{\pm{1.41}}} $   &  $ {{39.24}_{\pm{0.81}}} $  & $ {{58.33}_{\pm{1.43}}} $ & $ {{41.03}_{\pm{1.11}}} $ \\
               \algoGArnoldi{} & $\bf {{ 95.08}_{\pm{1.33}}} $  & $\bf {{ 94.95}_{\pm{1.13}}} $   &  $ \bf{{42.33}_{\pm{0.46}}} $  & $ {{68.92}_{\pm{1.29}}} $ & $ {{49.24}_{\pm{1.21}}} $ \\
    \hline\hline
    \end{tabular}%
    }
    %\vspace{0.1in}
%    \label{tab:FullHeto}
%\end{table}
%----------------------------------------------------------------

%--------------------------------------------------

%\begin{table}[ht!]
%    \caption{Accuracy of {\bf full-supervised} node classification on {\bf Large-Heterophilic} datasets.}
  \vspace{0.05in}
     \centerline{{\bf Fully-supervised} on {\bf Large Heterophilic} datasets}  
\resizebox{\columnwidth}{!}{%
    \centering
    \begin{tabular}{zyyyyy}
    \hline\hline
       {\bf Method} & {\bf roman-empire} & {\bf amazon-ratings} & {\bf minesweeper} & {\bf tolokers} & {\bf questions} \\
        \hline
        GCN& ${47.51}_{\pm { {0.25}}}$  &${43.65}_{\pm { {0.26}}}$ & ${72.44}_{\pm { {1.04}}}$&  $ {{74.96}_{\pm{0.92}}}$ & $ {{67.34}_{\pm{0.56}}}$ \\
        
        GAT& ${44.95}_{\pm { {1.34}}}$ & ${42.33}_{\pm{ {0.31}}}$ & ${73.51}_{\pm { {0.99}}}$ & ${ 72.18}_{\pm { {0.90}}}$ & ${69.23}_{\pm { {0.90}}}$ \\
        
        APPNP& ${49.34}_{\pm { {0.17}}}$   & ${42.59}_{\pm { {0.05}}}$ & ${68.74}_{\pm { {1.03}}}$ & ${69.69}_{\pm { {0.57}}}$ & ${59.00}_{\pm { {0.77}}}$ \\
        
        ChebNet& ${66.23}_{\pm { {1.04}}}$  &  $ {44.80}_{\pm { {0.43}}}$& ${80.12}_{\pm { {0.33}}}$ & $\bf{{80.26}_{\pm { {1.56}}}}$& ${70.67}_{\pm { {0.19}}}$ \\

        JKNet& ${43.55}_{\pm { {0.20}}}$  & ${43.95}_{\pm { {0.48}}}$& ${76.32}_{\pm { {0.24}}}$ & ${74.51}_{\pm { {0.67}}}$& ${59.55}_{\pm { {1.02}}}$\\
        
        GPR-GNN& ${63.64}_{\pm { {0.31}}}$ & ${45.14}_{\pm { {0.50}}}$ & ${81.89}_{\pm { {0.79}}}$ & ${70.84}_{\pm { {0.99}}}$ & ${68.02}_{\pm { {1.04}}}$ \\

        BernNet& ${63.83}_{\pm { {0.38}}}$ & ${44.20}_{\pm { {0.75}}}$ & ${77.90}_{\pm { {1.13}}}$ & ${71.23}_{\pm { {0.76}}}$ & ${66.19}_{\pm { {0.95}}}$ \\
        JacobiConv& ${62.24}_{\pm { {0.34}}}$ & ${41.02}_{\pm { {0.61}}}$ & ${79.95}_{\pm { {1.65}}}$ & ${69.17}_{\pm { {0.31}}}$ & ${67.09}_{\pm { {0.45}}}$ \\
        \hline\hline
             \algoArnoldi{} & $ {{ 61.77}_{\pm{0.35}}} $  & $ {{ 43.85}_{\pm{0.30}}} $  &  $ {{77.27}_{\pm{1.09}}} $ & $ {{75.90}_{\pm{0.97}}} $   & $ {{65.56}_{\pm{0.86}}} $  \\
               \algoGArnoldi{} & $\bf {{73.34}_{\pm{0.34}}} $  & $\bf {{ 49.58}_{\pm{0.07}}} $   &  $ \bf{{92.14}_{\pm{0.46}}} $ & $ {{77.34}_{\pm{0.49}}} $   & $ \bf{{74.16}_{\pm{0.85}}} $  \\
    \hline\hline
    \end{tabular}%
    }
     \vspace{-0.15in}
    %\vspace{0.1in}
   % \label{tab:FullHetoLarge}
\end{table}

\begin{figure*}[t!]
\centering
\includegraphics[scale = 0.55]{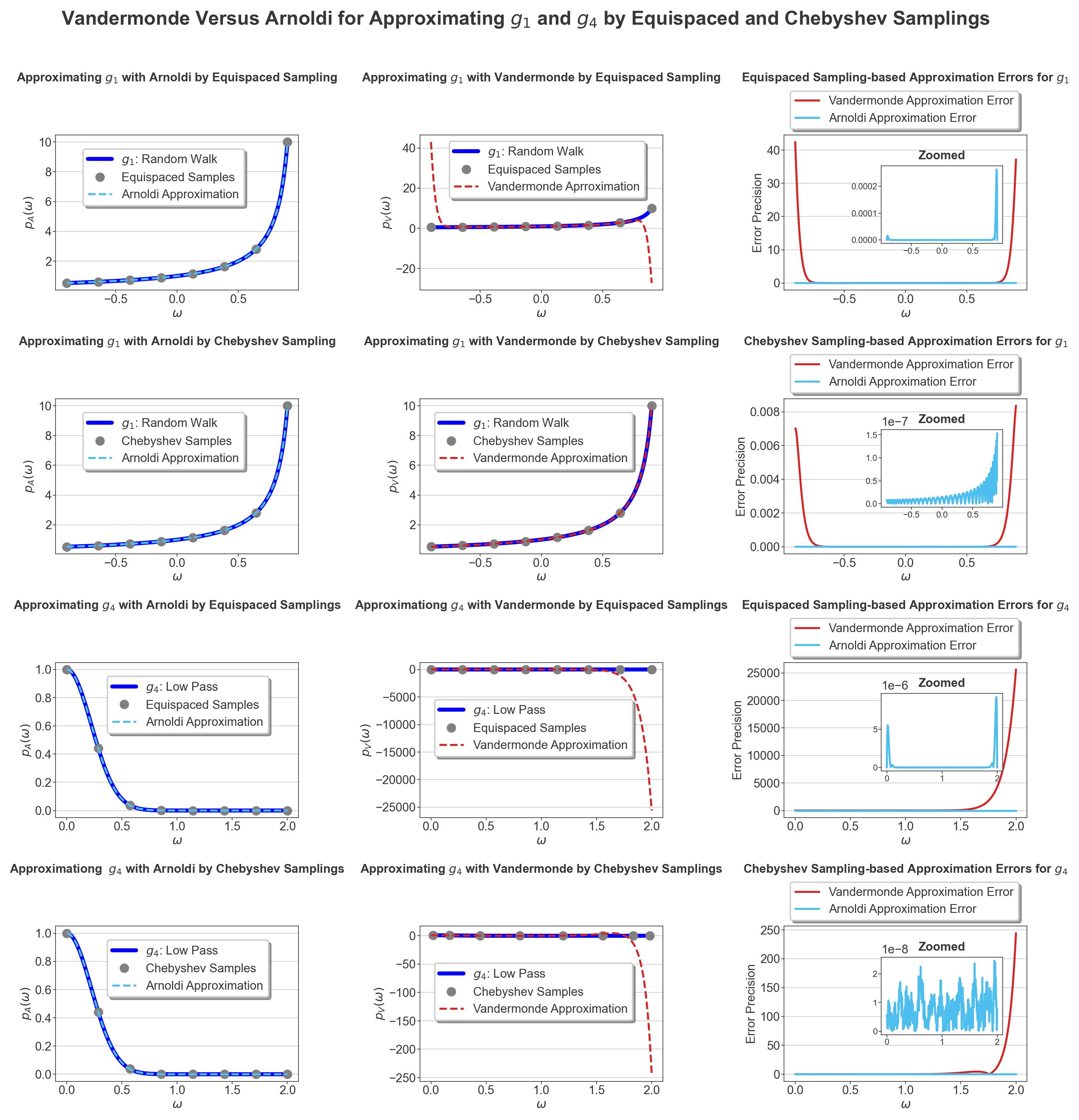}
\vspace{-0.1in}
\caption{{\bf{Approximating  ~\emph{Random Walk} and ~\emph{Low Pass}  filter functions using \emph{Equispace} vs. \emph{Chebyshev} sampling with direct solution to the Vandermonde system vs. Arnoldi orthonormalization.}} 
The first/ second columns show the approximation provided by Arnoldi decomposition/ direct solution, respectively, and the third column shows the error between approximations and actual filter functions. 
\label{fig:VanderAgenstArnoldiError}}
\vspace{-0.1in}
\end{figure*}

In this section, we systematically evaluate the performance of our proposed algorithms, \algoArnoldi{} and \algoGArnoldi{}
in utilizing explicit filters along with accurate polynomial approximation to improve node classification using Spectral GCNs. 
We first describe the datasets that contain five homophilic, five small-heterophilic , and five large heterophilic networks, as well as our experimental setup. 
Next, we compare the performance of \algoArnoldi{} and \algoGArnoldi{} against state-of-the-art algorithms GCN~\cite{kipf2016semi}, GAT~\cite{velivckovic2018graph}, APPNP~\cite{gasteiger2018predict}, ChebNet~\cite{xingjian2015convolutional},JKNet~\cite{xu2018representation}, GPR-GNN~\cite{chien2020adaptive}, BernNet~\cite{he2021bernnet}, and JacobiConv~\cite{wang2022powerful}) in both semi and fully supervised settings.
Finally, we assess the relationship between the numerical accuracy of the polynomial approximation and node classification performance of the resulting Spectral GCNs.
%'s \algoArnoldi{}'s accuracy in computing polynomial coefficients compared to the Vandermonde approach. We then evaluate these coefficients in semi-supervised node classification on three homophilic and three heterophilic datasets. Finally, 

\subsection{Datasets and Experimental Setup}
%As shown in the previous section,
% the proposed algorithm generates “accurate” coefficients in the sense that it is guaranteed to correctly fitting any filter by using any polynomial roots. For this reason, we here first focus
% on computation of accurate polynomial coefficients, and compare our proposed Arnoldi approach against
% other coefficient computation approach, Vandermonde algorithm. We then compare the node classification performance of the computed polynomial coefficients by Arnoldi and Vandermonde. Finally, we compare \algoGArnoldi{} against other state-of-the-art Spectral GCN algorithms on semi-supervised and full supervised node classification tasks.
We use fifteen real-world network datasets for our experiments. 
Five of these datasets are homophilic networks, which include three citation graph
datasets, Cora, CiteSeer and PubMed~\cite{yang2016revisiting},
and two Amazon co-purchase graphs, Computers and
Photo~\cite{shchur2018pitfalls}. 
The remaining five networks are heterophilic, which include Wikipedia graphs Chameleon and Squirrel~\cite{rozemberczki2021multi}, the Actor co-occurrence
graph, and the webpage graph Texas and Cornell from WebKB3~\cite{pei2019geom} as well as recently curated large heterophilic datasets, Roman-Empire, Amazon-Rating, Widesweeper, Tolokers, and Questions~\cite{platonov2022critical}. 
The statistics and homophily measures of these fifteen datasets are given in Table~\ref{tab:TableStats}. \emph{Adjusted Homophily} measures the association between being connected by an edge and having the same class 
label, corrected for class imbalance. 
\emph{Label Informativeness} measures the extent to which a neighbor's label provides information on the label of the node~\cite{platonov2024characterizing}.

In the following subsections, we report the results of three sets of experiments. 
For all settings, we perform cross-validation by hiding a fraction of labels (this fraction depends on the specific experiment as specified below) and assess performance using Accuracy (fraction of correctly classified nodes). 
{\tt{Widesweeper}}, {\tt{Tolokers}}, and {\tt{Questions}}  datasets have binary classes, thus we report area under ROC curve (AUROC) for these datasets.
We repeat all experiments five times and report the mean and standard deviation of accuracy.
 Hyper-parameter settings are presented in the Appendix.

\subsection{Comparison of Node Classification Performance of  \algoArnoldi{} and \algoGArnoldi{} to State-of-the-Art}
We compare our proposed algorithms, \algoArnoldi{} and \algoGArnoldi{}, against state-of-the-art Spectral GCNs on five homophilic and ten heterophilic real-world datasets for both semi and fully supervised node classification tasks. 
To assess semi-supervised node classification performance,
we use cross-validation by  randomly splitting the data into training/validation/test samples as (\%2.5 / \%2.5 /\%95).
To assess fully supervised node classification performance,
we use cross-validation by randomly splitting the data into training/validation/test samples as (\%60 / \%20 /\%20).
The results of these experiments are presented in Table~\ref{tab:SSHomo}
%, \ref{tab:SSHeto}, \ref{tab:SSHetoLarge}, \ref{tab:FullHomo}, ~\ref{tab:FullHeto}, and~\ref{tab:FullHetoLarge} 
with the best-performing method(s) highlighted in bold. 
Our methods, particularly \algoGArnoldi{}, demonstrate significant improvements in the performance of Spectral-GCNs, as seen in these tables.  
We observe that the performance improvement is particulary pronounced on heterophilic networks, highlighting the importance of using explicit filters for complex 
machine learning tasks.
To provide insights into which filter functions deliver best performance on which classification tasks, we provide the details of the best-performing filter in Table~\ref{tab:HyperSemi} in the Appendix.

\subsection{Value Added by Numerical Accuracy}

The results shown in the previous set of experiments demonstrate 
that the use of explicit filter functions and their polynomial approximation significantly enhances node classification performance of Spectral GCNs.
To gain insights into the value added by the numerical accuracy in this process, we first investigate the effect of the linear system solver and the sampling technique used for the
polynomial approximation on the accuracy of the approximation.
%Then we investigate how these differences in the accuracy of approximation translate to node classification performance for the resulting Spectral GCN.
%\subsubsection{Accuracy of Polynomial Approximation to Filter Functions}
Figure~\ref{fig:VanderAgenstArnoldiError} compares Arnoldi orhonormalization to direct solution of the Vandermonde system in approximating \emph{Random Walk} ($g_1$) and \emph{Low Pass} ($g_4$) filters using \emph{Equispaced} and \emph{Chebyshev} sampling. 
As seen in the figure, Arnoldi significantly reduces errors, particularly for complex filters. 
Furthermore, \emph{Chebyshev} sampling consistently outperforms \emph{Equispaced} sampling. 

%The Arnoldi-based method achieves unprecedented precision in approximating complex filters, demonstrating at least $10^{-6}$ error precision, even with \emph{Equispace} samples.

\begin{figure}[t]
\centering
\includegraphics[scale = 0.7]{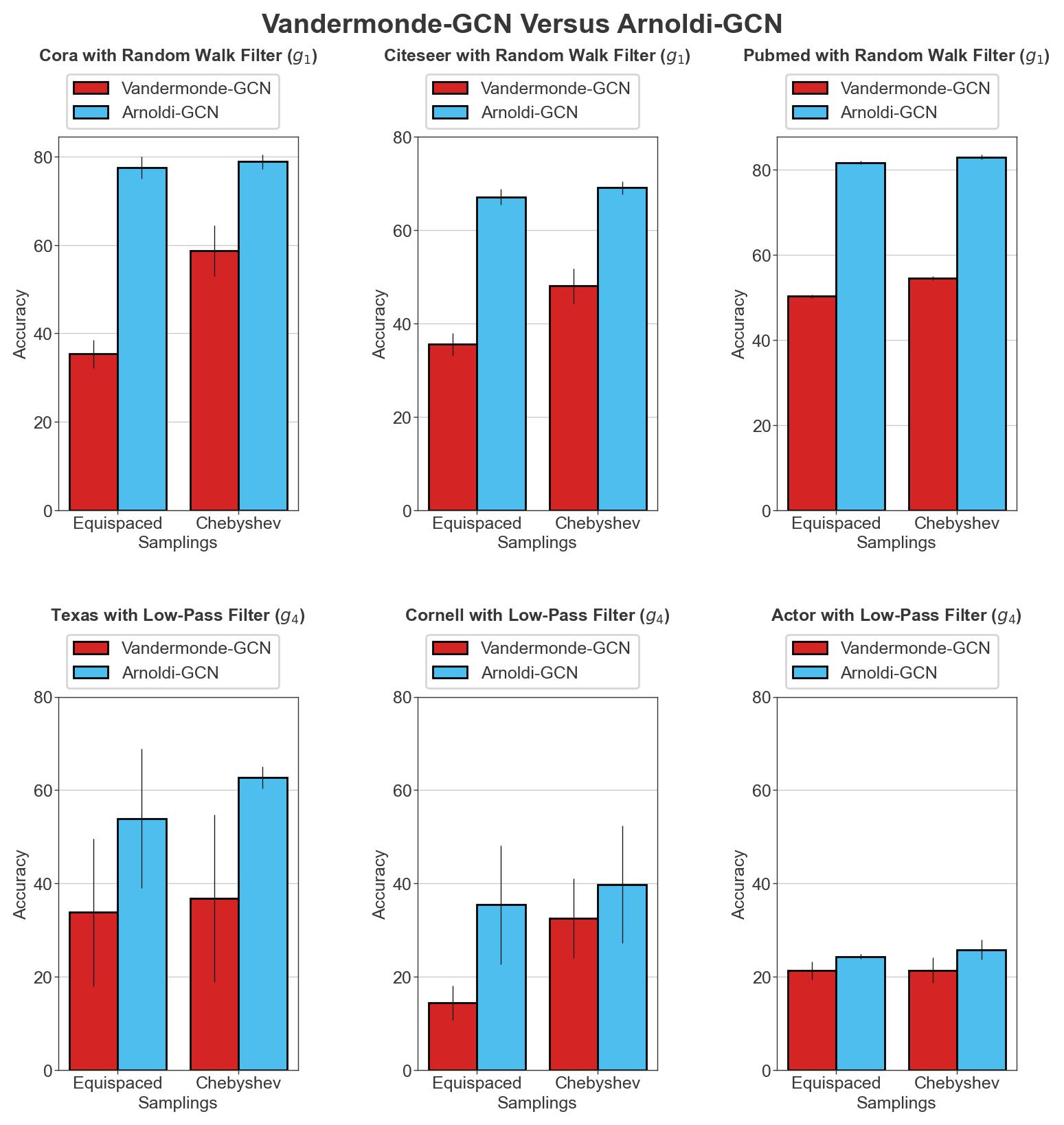}
\vspace{-0.1in}
\caption{ {\bf {Semi-supervised learning performance of Vandermonde and Arnoldi-generated approximations.}} 
%is evaluated on three homophilic datasets (Cora, Citeseer, Pubmed) and three heterophilic datasets (Texas, Cornell, Actor).}} 
The plots show the mean and standard deviation of node classification accuracy across five runs on six datasets. \label{fig:VanderAgenstArnoldi}}
\vspace{-0.25in}
\end{figure}

\subsection{From Approximation Quality to Classification Accuracy}
In the previous section, we observe that Arnoldi orthonormalization drastically improves the accuracy of polynomial approximations to filter functions.
We now assess how this numerical accuracy translates to the observed superior classification performance of the Spectral GCN.
For this purpose, we compare the node classification performance of the Spectral GCNs that utilize polynomials generated by direct solution of the Vandermonde system (\algoVandermonde{}) vs. 
Arnoldi orthonormalization (\algoArnoldi{}), using three homophilic and three heterophilic datasets. 
In these experiments, we use \emph{Random Walk} and \emph{Low Pass} filters and polynomial approximations computed using Equispaced and Chebyshev sampling.
As seen in Figure~\ref{fig:VanderAgenstArnoldi}, \algoArnoldi{} consistently and significantly outperforms  \algoVandermonde{}, demonstrating the importance of numerically stable polynomial approximations in training reliable GCN models. 
In addition, \emph{Chebyshev} sampling consistently outperforms \emph{Equispace} in semi-supervised node classification.
Integration of complex filters, e.g., \emph{Low-Pass}, delivers substantial improvements, especially for heterophilic datasets such as Texas and 
Cornell.

\subsection{Discussion}
Based on the comprehensive results presented in this section, we derive the following key insights:
\begin{itemize}
    \item \algoGArnoldi{} outperforms state-of-the-art algorithms on 12 out of 15 datasets in both semi-supervised and supervised node classification tasks.
    \item On 6 out of 10 heterophilic datasets, \algoGArnoldi{}'s performance improvement over other algorithms is highly significant (more than 10 standard deviations).
    %\item Across all datasets and classifications tasks, both \algoArnoldi{} and \algoGArnoldi{} deliver much smaller variance in classification accuracy than competing algorithms, i.e., they are more stable.
    \item On homophilic graphs, filters that assign more weight to local neighborhood of nodes deliver best performance. These include simple filters, which can be considered variations of random walks, as well as low-pass filters.
    \item On heterophilic graphs, self-depressed and neighbor-depressed filters compete with high-pass 
    filters. This is likely because they initialize coefficients to assign less weight to a node's self and immediate neighbors, allowing supervised learning to converge to desirable local optima.
    %\item While Chebyshev sampling outperforms other sampling methods with \algoArnoldi{} (fixed coefficients), the sampling method that delivers best classification accuracy is variable for \algoGArnoldi{} (learnt coefficients). Thus high quality approximation improves performance but does so more by helping the 
    %machine learning algorithm escape from undesirable local optima.
    \item On the Questions dataset, band-rejection filters outperform all other filters and competing algorithms by more than 10 standard deviations. This indicates that the graph has specific topological properties that cannot be captured solely by homophily and label informativeness. Additionally, this observation demonstrates that the ability to design and apply complex filters enhances the effectiveness of spectral GCNs on graphs with intricate topological characteristics.
\end{itemize}
%-------------------------------------------------------
% Subsection
%-------------------------------------------------------

%% file: Conclusion.tex
\section{Conclusion}
\label{seq:conclusion}

 The ill-conditioned nature of polynomial approximation to filter functions poses significant challenges for Spectral GCNs in defining suitable signal propagation.
 Our algorithm, \algoGArnoldi{}, overcomes this challenge by enabling numerically stable polynomial approximation of any filter function. 
 \algoGArnoldi{} excels in multi-class node classification across diverse datasets, showcasing its customizability with various filters and sampling methods available to choose from. 
 %his paper shows that filter function choice significantly influences GCN performance, and \algoGArnoldi{} consistently outperforms state-of-the-art algorithms when applied appropriately. In summary, \
 \algoGArnoldi{} marks a new direction in graph machine learning, enabling the explicit design and application of diverse filter functions in Spectral GCNs.

%% file: Appendix.tex
\section{Appendix}

\subsection{Proofs of Theorems}
In this subsection, we prove the theorems in the body of the paper. In the following proofs we assume $K = r$ and all the proofs are valid when $K \leq r$ since least square fitting is norm-wise backward stable~\cite{demmel1997applied}.

\begin{proof} [Proof of Theorem~\ref{theorem:appendix}]
Let $\krylov^{K}  = span \big\{ \qvec_{1}, \OmegaMatrix \qvec_{1}, \OmegaMatrix^2\qvec_{1}, \cdot, \cdot, \cdot, \OmegaMatrix^{K} \qvec_{1} \big\}$ be the Krylov space, where $\qvec_1 = \frac{\identityvec}{\norm{\identityvec}}$, and its column vectors be $\QMatrix$ generated by applying QR factorization on $\VMatrix$. Assume that $\OmegaMatrix^{K} \qvec_{1}$ is dropped from the Krylov space because QR factorization found an invariant subspace at point $K$. Let $\textbf{v}$ be an arbitrary  orthonormal vector to $\OmegaMatrix^{K} \qvec_{1}$ as $\ortvec$, i.e., $\ortvec$ is perpendicular to $\OmegaMatrix^{K} \qvec_{1}$. Since this $\ortvec$ is perpendicular to $\OmegaMatrix^{K} \qvec_{1}$ that is dropped from $\krylov^{K}$, we can express this vector as a linear combination of all previous vectors in $\krylov^{K}$. That is, for some $\beta_i \in (0,1)$, we have 
$ \ortvec = \beta_1 \qvec_1 + \beta_2 \OmegaMatrix\qvec_1 + \cdot \cdot \cdot + \beta_{K} \OmegaMatrix^{K} \qvec_{1}$. 
Clearly, we can state the last equation as a monic polynomial, i.e.,
%\[
$\ortvec = p_{K} (\OmegaMatrix) \qvec_1$.
%\]
Let $\mathcal{P}_j$ be the family of all monic polynomials, then:
\begin{align}
\norm{\ortvec} & = \min\limits_{p_j(\omega) \in \mathcal{P}_j}\norm{p_{K} (\OmegaMatrix) \qvec_1}&\\
     & \leq  \min\limits_{p_j(\omega) \in \mathcal{P}_j} \max\limits_{-\alpha \leq \omega \leq \alpha} |p_{K} (\omega)| \norm{\qvec_1} \\
     & = \min\limits_{p_j(\omega) \in \mathcal{P}_j} \max\limits_{-\alpha \leq \omega \leq \alpha} |p_{K} (\omega)| 
\end{align}

The last equation suggests that choose such a j $\omega$s from $[-\alpha, \alpha]$, $-\alpha \leq \omega_1, \omega_2, \cdot, \cdot, \cdot, \omega_j \leq \alpha$, such that following values are as small as possible:
\[
\max\limits_{-\alpha \leq \omega \leq \alpha} |p_{K} (\omega)| = \max\limits_{-\alpha \leq \omega \leq \alpha} |(\omega-\omega_1) (\omega-\omega_2) \cdot \cdot \cdot (\omega - \omega_j)|
\]

From Approximation Theory, we know that Chebyshev polynomial first kind's roots satisfy the above equation because from cosine definition of Chebyshev polynomial, the maximum value of $\mathcal{T}_{K} (\frac{\omega}{\alpha}) =1$. To use this property, let us define re-scaled Chebyshev polynomial as: 
\[
\mathcal{T}_{K} (\frac{\omega}{\alpha}) = \frac{1}{2^{K-1}} {\alpha}^K \omega^K + \cdot \cdot \cdot
\]
Since $\omega$s are chosen as the roots of the above polynomial:
\[
|(\omega-\omega_1) (\omega-\omega_2) \cdot \cdot \cdot (\omega - \omega_j)| = \frac{1}{2^{K-1}} {\alpha}^K |\mathcal{T}_{K} (\frac{\omega}{\alpha})|
\]
Then,
\[
 \min\limits_{p_j(\omega) \in \mathcal{P}_j} \max\limits_{-\alpha \leq \omega \leq \alpha} |(\omega-\omega_1) (\omega-\omega_2) \cdot \cdot \cdot (\omega - \omega_j)| \leq \frac{1}{2^{K-1}} {\alpha}^K
\]
Thus, 
$
\norm{\ortvec} \leq (1/2^{K-1}) {\alpha}^K 
$.
Using the definition of Frobenius norm  of a low rank matrix, and using the fact that limit of matrix goes zero after convergence, we have:

\[
\norm{\QMatrix^{\dagger}}_2 = \sigma_r \geq 2^{r-1} {\big(\frac{1}{\alpha}\big)}^r ~\mathrm{and}~ \norm{\QMatrix}_2 \geq 1
\]
Thus, we have:
\[
\kappa (\VMatrix) = \norm{\VMatrix}_2 \norm{\VMatrix^{\dagger}}_2 \geq \sigma_r \geq 2^{r-1} {\big(\frac{1}{\alpha}\big)}^r
\]

Now, we need to show that $\VMatrix$ is still ill-conditioned when $\omega_k \geq 1$ for $ k =1 , \cdot, \cdot, \cdot, r$ to establish the result for range $(\varepsilon, 2)$ (since we have shown result holds for (-1,1)). To prove this, let $\hat{\qvec}_1 = \OmegaMatrix^{r-1} \qvec_1$ and $\widehat{\OmegaMatrix} = \frac{1}{\OmegaMatrix}$. Now, let $\widehat{\VMatrix}$ be reversely ordered version of the original $\VMatrix$. Clearly, these matrices have the same condition number, i.e., $\kappa(\VMatrix) = \kappa (\widehat{\VMatrix})$. Furthermore: 
\[
\widehat{\VMatrix} = [\hat{\qvec_1},\widehat{\OmegaMatrix} \hat{\qvec_1}, \cdot, \cdot, \cdot, \widehat{\OmegaMatrix}^{K-1} \hat{\qvec_1}]
\]
Now, all elements of $\widehat{\OmegaMatrix}$ are in interval [-1,1], thus (from the previous establishment for (-1,1)) we have: 
\[
\kappa (\widehat{\VMatrix}) \geq 2^{r-2}
\]
Therefore, condition number of $\VMatrix$ or $\widehat{\VMatrix}$ , which is computed by QR factorization on $\VMatrix$ is exponentially unbounded and thus coefficients generated by this $\VMatrix$ are inaccurate in Equation~\ref{eq:weights}.
\end{proof}

\begin{proof}[Proof of Theorem~\ref{thm:theorem2}]
By using the definition of Vandermonde matrix, $\VMatrix = [\identityvec, \OmegaMatrix \identityvec, \cdot, \cdot, \cdot, \OmegaMatrix^{K-1} \identityvec]$, $\VMatrix^{(*)} = [\qvec_1, \OmegaMatrix \qvec_1, \cdot, \cdot, \cdot, \OmegaMatrix^{K-1} \qvec_1]$, where $\qvec_1 = \frac{\identityvec}{\norm{\identityvec}}$.
Let $\identityvec_{i} \in \mathbb{R}^{K-1}$ be i-th unit vector, i.e., only i-th entry is 1 and rest zero. We prove the theorem by induction. Let $i\leq K-1$, then for $i = 1$, since $\QMatrix = [\qvec_1, \qvec_2, \cdot, \cdot, \cdot,\qvec_{K} ]$ we have,
\begin{equation}
    \VMatrix^{(*)} \identityvec_1 = \QMatrix \RMatrix \identityvec_1 = \QMatrix \identityvec_1 = \qvec_1
    \label{eq:induction1}
\end{equation}

Now, assume $i = 2$, then we have,
\begin{equation}
    \VMatrix^{(*)} \identityvec_2 = \OmegaMatrix \qvec_1 = (\OmegaMatrix \QMatrix - \AlmostZero)\identityvec_1 = \QMatrix \TMatrix \identityvec_1 = \QMatrix \RMatrix \identityvec_2 
    \label{eq:induction2}
\end{equation}

Assuming that the claim holds for $i = K-1$, we have:

\vspace{-0.1in}
\begin{equation}
\begin{aligned}
    \VMatrix^{(*)} \identityvec_{K-1} &
    =~\OmegaMatrix^{K-1} \qvec_1
    =~(\OmegaMatrix^{K-1} \QMatrix - \AlmostZero)\identityvec_1 \quad\\
    & =~\QMatrix \TMatrix^{K-1} \identityvec_1
     =~\QMatrix \RMatrix \identityvec_{K-1} 
    \end{aligned}
    \label{eq:inductionjminus1}
\end{equation}

Now, we need to show Equation~\ref{eq:inductionjminus1} holds for $i = K$. First notice that from the assumption, we have 
$\OmegaMatrix^{K}\qvec_1 = \OmegaMatrix (\OmegaMatrix^{K-1} \qvec_1) = \OmegaMatrix (\QMatrix \TMatrix^{K-1} \identityvec_1  ) = (\OmegaMatrix \QMatrix) \TMatrix^{K-1} \identityvec_1 $,
so we can write:
\begin{equation}
\begin{aligned}
\OmegaMatrix^{K}\qvec_1 & =(\OmegaMatrix \QMatrix) \TMatrix^{K-1} \identityvec_1  = (\QMatrix \TMatrix - \AlmostZero)\TMatrix^{K-1} \identityvec_1 & \\
     & = \QMatrix \TMatrix^{K}\identityvec_1 - \AlmostZero \TMatrix^{K-1} \identityvec_1 
      = \QMatrix \TMatrix^{K}\identityvec_1 
     = \QMatrix \RMatrix \identityvec_{K+1} &
    \end{aligned}
\end{equation}
Finally, we have $\VMatrix^{(*)} \identityvec_{K+1} = \QMatrix \RMatrix \identityvec_{K+1}$. This, prove that $\norm{e} \times \QMatrix\RMatrix$ is QR of Vardermonde matrix. 
This completes the proof.

\end{proof}

\begin{proof} [Proof of Theorem~\ref{theorem:main}]
Assume $K = r$, then $\QMatrix$ is full rank and $\mathrm{rank}\ ({\QMatrix}) = r$ since $\omega$s are unique. When $K \leq r$ the proof still holds since least square fitting is norm-wise backward stable~\cite{demmel1997applied}. Ignoring the constant term, $r$, from the construction, $\QMatrix$ is orthonormal. Thus, we have $\QMatrix^T = \QMatrix^{-1}$ since $\QMatrix \QMatrix^T = \IMatrix$. Therefore, in theory, $\kappa (\QMatrix) = 1$. However, in practice, due to machine round-off error, this number might change, especially, when Ritz values, i.e., eigenvalues of $\TMatrix$, converge to eigenvalues of $\OmegaMatrix$. To illustrate this point, let $\varepsilon$ be machine epsilon and let us represent floating point by ${fl}$, i.e., stored version of $\QMatrix$ in the machine is ${{fl}} (\QMatrix)$.
We want to measure the difference between $\QMatrix \QMatrix^{-1}$ and $\IMatrix$. That is, we want to compute:
\[{ {fl}} (\QMatrix \QMatrix^{-1} - \IMatrix)\]
Using Golub and Van~\cite{golub2013matrix} (2.4.15) and (2.4.18), we have:
\[{ {fl}} (\QMatrix \QMatrix^{-1} - \IMatrix) = (\QMatrix \QMatrix^{-1} - \IMatrix) + \EMatrix_{1}~~~~~~~~~ \EMatrix_{1} \leq \varepsilon \norm{\QMatrix \QMatrix^{-1} - \IMatrix)} \]
and
\[{ {fl}} (\QMatrix \QMatrix^{-1}) = \QMatrix \QMatrix^{-1} + \EMatrix_{1}~~~~~~~~~~~~~~~ \EMatrix_{2} \leq \varepsilon \times r \norm{\QMatrix \QMatrix^{-1} } + \mathcal{O}(\varepsilon^2)\]
By using the above floating point analysis for matrix multiplication and addition,  $\kappa (\QMatrix)$ perturbs from 1 around 0.01 due to $\mathcal{O}(\varepsilon^2)$ error term. Thus, we have
\[\kappa (\QMatrix) \approx 1.01\]
\end{proof}

% \subsection{Applying QR Factorization on Vandermonde Matrix Produces Illegal Polynomial Coefficients }
% Our first attempt to compute the polynomial coefficient relies on constructing and solving a Vandermonde Linear System by applying QR factorization on this Vandermonde matrix. However, the following theorem proves that this approach produces illegal coefficients because the condition number of $\VMatrix$ is exponentially bounded. 

 \begin{table}[t!]
    \centering
     \caption{Hyperparameters used for \algoArnoldi{} and \algoGArnoldi{} for node classification.}
    \label{tab:HyperSemi}
    \resizebox{\columnwidth}{!}{
    \begin{tabular}{zyyxx}
    \hline\hline
        {\bf Datasets } & {\bf Samplings} & {\bf Filter}  & \makecell{\bf Propagation layer \\
\bf learning rate} & \makecell{\bf Propagation layer \\
\bf droupout rate} \\
        \hline
        \multicolumn{5}{c}{{\em{Semi-Supervised classification}}}\\
        \hline
        Cora & Jacobi  & $g_0$ (Scaled RW)  &  0.001 & 0.7\\
        
        Citeseer& Jacobi  & $g_2$ (Self-Depressed RW)&  0.002 & 0.6  \\
        
        Pubmed & Chebyshev & $g_3$ (Neighbor-Depressed RW) &  0.001 & 0.1 \\
        
        Photos & Chebyshev  & $g_0$ (Scaled RW) &  0.002 & 0.1\\
        Computers & Equispace  & $g_2$  (Self-Depressed RW) &  0.002 & 0.1  \\
        Texas & Legendre  & $g_4 $ (High-Pass) &  0.01& 0.2 \\
        
        Cornell & Jacobi  & $g_7 $ (Band-Rejection) &  0.05 & 0.9  \\
        
        Actor & Jacobi  & $g_7 $ (High-Pass) &  0.001 & 0.5\\
        
        Chameleon & Legendre  & $g_1 (RW)$  &  0.05 &  0.8 \\
        Squirrel & Equispace  & $g_0$ (Self-Depressed RW)  &  0.01 &  0.2\\
         roman-empire&  Chebyshev & $g_5$ (High-Pass)  &  0.01 & 0.2\\
        
        amazon-ratings& Legendre& $g_5$ (High-Pass)&  0.01 & 0.3 \\
        
        minesweeper & Jacobi & $g_3$ (Neigbor Depressed-RW) &  0.01 & 0.1 \\
        
        tolokers & Jacobi  & $g_5$ (High-Pass) &  0.002 &  0.5 \\
        questions & Chebyshev & $g_7$  (Band-Rejection) &  0.002 &  0.1 \\
    \hline
    %\hline
    %\end{tabular}}
    
    %\vspace{0.05in}
   
%\end{table}

%\label{eq:results}
%\begin{table}[t!]
 %   \centering
  %    \caption{Hyperparameters for \algoGArnoldi{} on real-world datasets for Full-Supervised classification.}
   % \label{tab:HyperFull}
    %\resizebox{\columnwidth}{!}{
    %\begin{tabular}{zyyxx}
     \hline
        \multicolumn{5}{c}{{\em{Fuly-Supervised classification}}}\\
    \hline
        {\bf Datasets } & {\bf Samplings} & {\bf Filter} & \makecell{\bf Propagation layer \\
\bf learning rate}& \makecell{\bf Propagation layer \\
\bf droupout rate} \\
        \hline
        Cora & Jacobi  & $g_3$ (Neighbor-Depressed RW) &  0.05 & 0.4 \\
        
        Citeseer& Chebyshev  & $g_3$ (Neighbor-Depressed RW) &  0.05 & 0.4 \\
        
        Pubmed & Equispace & $g_5$ (Low-Pass) &  0.001 & 0.1 \\
        
        Photos & Equispace  & $g_1$ (RW)&  0.002 & 0.5 \\
        Computers & Chebyshev  & $g_5$ (Low-Pass)&  0.002 &  0.5  \\
        Texas & Legendre  & $g_2 $ (Self-Depressed RW) &  0.05 &  0.4 \\
        
        Cornell & Chebyshev  & $g_2 $ (Self-Depressed RW)&  0.05 & 0.4\\
        
        Actor & Equispace  & $g_2 $ (Self-Depressed RW) &  0.05 & 0.4 \\
        
        Chameleon & Legendre  & $g_0 $  (Scaled RW)&  0.05 & 0.5  \\
        Squirrel & Jacobi  & $g_2 $ (Self-Depressed RW)&  0.05& 0.3  \\
        roman-empire&  Chebyshev & $g_5$ (High-Pass)  &  0.001 &  0.2  \\
        
        amazon-ratings& Chebyshev& $g_5$ (High-Pass)&  0.002 & 0.1  \\
        
        minesweeper & Jacobi & $g_3$ (Neighbor Depressed-RW) &  0.01 & 0.1 \\
        
        tolokers & Equispace  & $g_2$ Neighbor Depressed-RW) &  0.002& 0.5  \\
        questions & Chebyshev & $g_7$  (Band-Rejection) &  0.002 & 0.1  \\
    \hline\hline
    \end{tabular}}

     \vspace{-0.2in}
\end{table}

\subsection{Hyper-parameters}

In the body of the paper, we perform three sets of experiments in which the following hyperparameters are used:

\textbf{Experiment 1.} For these experiments, we use two hyper-parameters that are -- the number of samples, $r$, used to approximate the filter functions, and dimension of $\QMatrix$, $K$. We set these two parameters to $r, K = 40$.
    %********************************************************** 

\textbf{Experiment 2.} In this experiment, we use \algoArnoldi{} in Equation~\ref{eq:ArnoldiSpectralGCN} in which coefficients are computed by Vandermonde~\ref{eq:Vandermonde} and Arnoldi in Algorithm~\ref{algo:Arnoldi}.
% \[
% \YMatrix = \text{softmax}\big( \ZMatrix\big), \ZMatrix_{\text{ARNOLDI}} = \sum_{k = 0}^K a_{k} \HMatrix^ {(k)}, \HMatrix^ {(0)} = f_{\theta} (\XMatrix), 
% \HMatrix^ {(k)} = \begin{cases} 
%       \widetilde{\PMatrix}  \HMatrix^ {(k-1)}, &~~~ \omega \in [-0.9, 0.9] \\
%       \widetilde{\LMatrix}  \HMatrix^ {(k-1)}, &~~~ \omega \in [\varepsilon,2]
%    \end{cases}
% \]
For these experiments, we limit ourselves to semi-supervised node classification task and  consider the random split into training/validation/test samples, which is a sparse splitting (\%2.5 / \%2.5 /\%95). For fair comparison, we use the same filter functions for both Vandermonde and Arnoldi. Namely, for homophilic datasets, we use \emph{RW filter function}, i.e., $g_{RW} (\omega) =\frac{1}{1-\omega}$, where $\omega \in [-0.9, 0.9]$, and for heterophilic datasets, we use \emph{low pass filter function}, i.e.,  $g_{LP} (\omega) =exp\{-10\omega\}$, where $\omega \in [\varepsilon, 2]$, $\varepsilon$ is $10^{-5}$. Furthermore,  we choose number of samples and dimension of $\QMatrix$ as $r,K = 10$ and use a 2-layer (MLP) with 64 hidden units for the NN component. For other hyperparameters, we fix \{\emph{learning rate} / \emph{weight decay} / \emph{dropout}\} to \{0.002 /0.0005/0.5\}, respectively. 
    
    %**********************************************************
\textbf{Experiment 3.}
In these experiments, we use eight representative filter functions in Equations~\ref{eq:g0}-~\ref{eq:g7} and approximate them by using four representative polynomial samples (Equispaced, Chebyshev, Legendre, and Jacobi roots) in Equations~\ref{eq:p0}-~\ref{eq:p3}. We use bound [-0.9, 0.9] for simple filter functions ($g_0 - g_3$) and $[10^{-5}, 2]$ for complex filter functions ($g_4 - g_7$). 
For semi-supervised node classification tasks we again randomly split the data into training/validation/test samples as (\%2.5 / \%2.5 /\%95); and for full-supervised node classification tasks, we randomly split the data into training/validation/test samples as (\%60 / \%20 /\%20). Again, we choose number of samples and dimension of $\QMatrix$ as $r,K = 10$ and a 2-layer MLP with 64 hidden units are used, with weight decay = 0.0005. Learning rate is tuned within {0.001, 0.002, 0.01, 0.05} and dropout within {0.1 to 0.9}. The best combinations are in the tables.

For fair comparison, we use the best hyperparameters from state-of-the-art Spectral GCN algorithms (GCN, GAT, APPNP, ChebNet, JKNet, GPR-GNN, BernNet,  JacobiConv). For random walk length in APPNP, ChebNet, GPR-GNN, BernNet, and JacobiConv we use 
K =10. We generate polynomials using Python's \emph{poly1d}.
The hyperparameter combinations that are used to generate the results in Table~\ref{tab:SSHomo} are shown in Table~\ref{tab:HyperSemi}.